\documentclass[11pt]{article}
\usepackage[preprint]{acl}
\usepackage{times,latexsym}
\usepackage[T1]{fontenc}
\usepackage[utf8]{inputenc}
\usepackage{microtype,iftex}
\ifPDFTeX
  \usepackage{inconsolata}
  \pdfmapfile{+zi4.map}
\else
  
\fi
\usepackage{graphicx,booktabs,array,multirow,placeins}
\usepackage{algorithm,algpseudocode,cuted}
\newsavebox{\appendixsettingsbox}
\newlength{\appendixsettingsheight}
\newlength{\appendixsettingsrowextra}
\usepackage{amsmath,amssymb,mathtools,amsthm,xspace}
\mathtoolsset{showonlyrefs}

\newcommand{\MOPD}{\textsc{MOPD}\xspace}
\newcommand{\Dthree}{\textsc{D}$^3$-\textsc{MOPD}\xspace}
\newcommand{\NashMOPD}{\textsc{Nash-MOPD}\xspace}
\newlength{\MethodColWidth}
\newlength{\BenchmarkColWidth}
\newcolumntype{L}{>{\raggedright\arraybackslash}m{\MethodColWidth}}
\newcolumntype{B}{>{\centering\arraybackslash}m{\BenchmarkColWidth}}
\newcommand{\BenchmarkHeader}[2]{\shortstack{#1\\{\fontsize{8}{9}\selectfont #2}}}
\newcommand{\BenchmarkTableStyle}{\fontsize{9}{10.5}\selectfont
  \setlength{\tabcolsep}{1pt}\setlength{\MethodColWidth}{64pt}\setlength{\BenchmarkColWidth}{\dimexpr
    (\textwidth-\MethodColWidth-24\tabcolsep)/12\relax}\renewcommand{\arraystretch}{1.16}}
\newtheorem{theorem}{Theorem}[section]
\newtheorem{proposition}[theorem]{Proposition}
\newtheorem{lemma}[theorem]{Lemma}

\theoremstyle{definition}

\title{Nash-MOPD: Optimizing Task Trade-offs in\\Multi-Teacher On-Policy Distillation}
\author{
Siyuan Gan$^{1}$,
Bo Yang$^{1}$,
Linjian Meng$^{\dag 2}$,
Boyan Wang$^{\dag 1}$, \\
{\bf Xiaolong Qi$^{3}$, Jing Huo$^{1}$, Yang Gao$^{1}$} \\
$^1$State Key Laboratory of Novel Software Technology, Nanjing University, Nanjing, China \\
$^2$Shanghai Artificial Intelligence Laboratory, Shanghai, China\\
$^3$Yili Normal University, Yining, China\\
gansiyuan@smail.nju.edu.cn, menglinjian@pjlab.org.cn, boyanwang@nju.edu.cn\\
}
\makeatletter
\def\blfootnote#1{\begin{NoHyper}\xdef\@thefnmark{}\@footnotetext{#1}\end{NoHyper}}
\makeatother
\usepackage[skins,raster]{tcolorbox}
\usepackage{listings}
\hypersetup{pdftitle={Nash-MOPD: Optimizing Task Trade-offs in Multi-Teacher On-Policy Distillation},
  pdfauthor={Siyuan Gan, Bo Yang, Linjian Meng, Boyan Wang, Xiaolong Qi, Jing Huo, Yang Gao}}
\begin{document}
\maketitle
\begin{abstract}
Multi-teacher on-policy distillation (\MOPD) aggregates the capabilities of
multiple domain-expert teachers into a single student model.
However, all domains in MOPD share the student's finite budget of
optimization steps and rollouts.
Fixed uses constant domain sampling ratios based on corpus sizes;
by contrast, \Dthree adjusts the ratios using the teacher--student reverse-KL
divergence and its recent decrease.
However, favoring large recent decreases can direct more training toward
domains that already improve quickly, while domains that improve more
slowly may receive too few examples.
To address inflexible allocation and the risk of giving insufficient
training to domains that improve more slowly, we propose \NashMOPD,
an adaptive domain scheduler based on a local minimax game using existing
reverse-KL statistics.
The equilibrium maximizes the smallest predicted proportional reduction
under the local model and gives more prompts to domains that need more
examples for comparable improvement.
Consequently, \NashMOPD responds to changing learning progress and supports
domains that improve more slowly without extra teacher queries or rollouts.
At step 100, \NashMOPD exceeds both baselines' final mathematics and code
checkpoint scores on DeepSeek-R1-Distill-Qwen-1.5B and reaches Fixed's on
Qwen3-8B-Base. Furthermore, it leads eight of nine reported metrics and is
best or tied best on all six mathematics and code accuracy metrics,
respectively.
\end{abstract}
\blfootnote{$\dag$ Corresponding authors.}

\begin{figure*}[t]
\centering
\includegraphics[width=\textwidth]{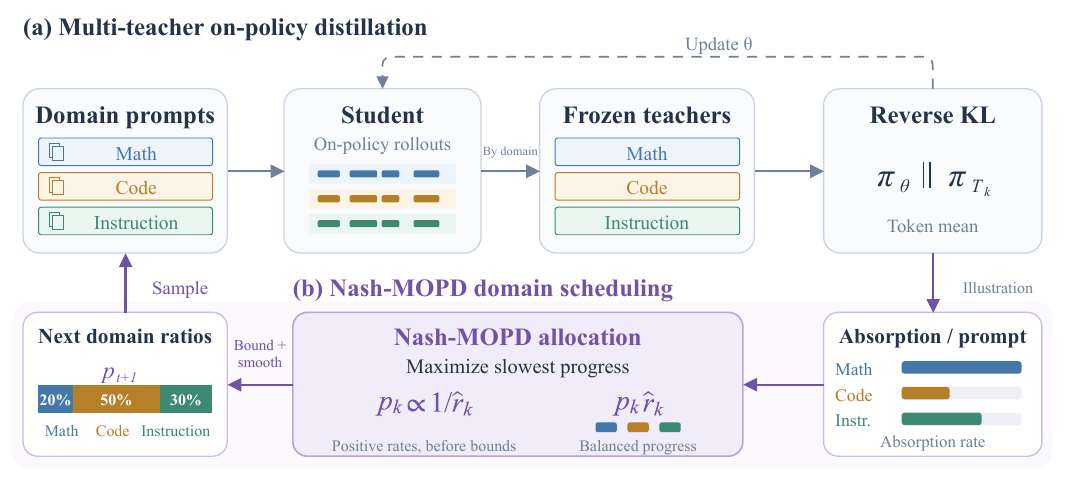}
\caption{Optimizing the multi-task trade-off with \NashMOPD. Frozen
domain-expert teachers supervise student rollouts; exposure-normalized
progress estimates guide prompt allocation across domains.}
\label{fig:method}
\vspace{-6pt}
\end{figure*}

\section{Introduction}
Multi-teacher on-policy distillation (\MOPD) aggregates the capabilities of
multiple domain-expert teachers into a single student model, which is widely
used in frontier model post-training
\citep{zeng2026glm5,deepseek2026v4,kimi2026k3,xiaomi2026mimov26}.
During training, the student generates rollouts from prompts sampled across
domains. The corresponding teacher then provides token-level supervision
at student-visited states, and the student is updated to minimize the reverse
KL divergence from its next-token distribution to the teacher's
\citep{ma2026mopd,meng2026unbiasedtopk}.
However, all domains in MOPD share the student's finite budget of
optimization steps and rollouts. As a result, increasing one domain's
sampling ratio leaves fewer training examples for the others.
MOPD therefore entails a multi-task trade-off, since the domain mixture
must balance competing learning demands to develop diverse capabilities
together. This allocation consequently shapes both convergence speed and
the final accuracy achievable within the training budget.

Effective use of this shared budget requires accounting for differences in
how the student learns across domains. As training proceeds, the student
may learn some domains quickly while requiring more examples to learn
others. Fixed uses constant domain sampling ratios based on corpus sizes,
so it cannot adjust the allocation to these differences or their changes
over time. By contrast, \Dthree adjusts the ratios using the teacher--student
reverse-KL divergence and its recent decrease \citep{sun2026d3mopd}.
However, favoring large recent decreases can direct more training toward
domains that already improve quickly, while domains that improve more
slowly may receive too few examples. We therefore need an adaptive
scheduler that responds to changes in the student's learning progress and
provides sufficient training for domains that improve more slowly, so as
to improve both convergence speed and final accuracy under the same
training budget.

To address inflexible allocation and the risk of giving insufficient
training to domains that improve more slowly, we propose \NashMOPD,
an adaptive domain scheduler based on a local minimax game.
At each allocation update, \NashMOPD uses existing reverse-KL statistics
to estimate how much of the remaining teacher--student difference each
domain reduces per sampled prompt. It then solves a local minimax game
in which the sampler chooses domain sampling ratios and the adversary
selects the domain with the smallest predicted proportional reduction.
The equilibrium maximizes this minimum progress under the local model
and gives more prompts to domains that need more examples for comparable
improvement. Sampling bounds preserve domain coverage, while smoothing
keeps successive allocations stable. Consequently, \NashMOPD responds
to changing learning progress and supports domains that improve more
slowly without extra teacher queries or rollouts. Moreover, this allocation
has a clear interpretation, since each domain's sampling ratio reflects
its estimated remaining training demand under the local model.

To evaluate the effectiveness of \NashMOPD, we conduct experiments
with DeepSeek-R1-Distill-Qwen-1.5B across mathematics, code, and
instruction following, and with Qwen3-8B-Base across mathematics and
code. All methods use the same training pipeline and budget.
On the DeepSeek-R1-Distill-Qwen-1.5B student, \NashMOPD exceeds both baselines' final mathematics
and code checkpoint scores at step 100. On the Qwen3-8B-Base student, it reaches
Fixed's final mathematics and code performance at step 100, using 60
and 80 fewer optimization steps than Fixed and the original \Dthree
baseline described in Section~\ref{sec:experiments}, respectively.
Furthermore, \NashMOPD leads eight of nine reported metrics on the
DeepSeek-R1-Distill-Qwen-1.5B student and is best or tied best on all six mathematics and code
accuracy metrics on the Qwen3-8B-Base student. These results demonstrate that
\NashMOPD outperforms both baselines in convergence speed and overall
final accuracy across domains.

\section{Related Work}
\paragraph{On-policy distillation and capability integration.}
Reverse-KL and generalized knowledge distillation supervise students on
their own responses \citep{gu2024minillm,agarwal2024gkd}. MOPD extends this
to multiple teachers, routing rollouts to the corresponding domain teacher
\citep{ma2026mopd}. We preserve its token-level supervision and adapt the
probability of choosing each prompt domain.

\paragraph{Adaptive mixtures and multitask optimization.}
GradNorm and related approaches adapt task weights from gradient magnitude,
loss evolution, or validation behavior \citep{chen2018gradnorm,
guo2018dynamic,sundararaman2021sampling}. DoReMi optimizes domain weights
using a proxy-model objective \citep{xie2023doremi}. \Dthree is the
closest setting: it repurposes per-domain reverse-KL to adjust the domain mixture \citep{sun2026d3mopd}. Its composite signal combines remaining gap and descent velocity;
our inverse-rate rule estimates exposure needed for completion. Our
original \Dthree baseline follows the published domain scheduling rule for allocating prompts across domains.

\paragraph{Games, fairness, and online feedback.}
Pareto multitask methods and Nash bargaining combine task gradients
\citep{sener2018mto,navon2022bargaining}. Nash learning from human feedback
uses a zero-sum preference game \citep{munos2024nash}. Here the adversary
selects a domain bottleneck and the learner controls domain sampling ratios, so the
players and payoff differ. GiGPO uses repeated states across existing
trajectories to obtain finer credit signals for agent training
\citep{feng2025gigpo}; our controller similarly consumes observations
already produced by training, but acts at the domain-mixture level. Neither
gradient bargaining nor trajectory credit assignment directly supplies the
exposure-based completion objective considered here.

\begin{algorithm*}[t]
\normalsize
\caption{\NashMOPD: online domain scheduling}
\label{alg:nash-controller}
\begin{algorithmic}[1]
\Require Domain pools $\{\mathcal D_k\}$, frozen teachers $\{\pi_{T_k}\}$, student $\pi_\theta$;
batch size $B$; $S$ updates
\Require Sampling bounds $(\ell,u)$; warm-up $w$; scheduling interval $h$;
smoothing $\alpha$; jitter $\gamma$
\State Initialize $p$ from usable-row proportions and $s_{\mathrm{last}}\gets0$
\State Initialize pool permutations/cursors, prompt exposures, reference KL scales, and EMA/gap histories
\For{$s=1,\ldots,S$}
\State Draw $\eta_k\sim\mathcal U[-\gamma,\gamma]$; set
    $\widetilde p_k\gets p_k(1+\eta_k)/\sum_j p_j(1+\eta_j)$
\State $b\gets\Call{LargestRemainder}{B\widetilde p}$; draw $b_k$ prompts per domain, reshuffling exhausted pools
\State Shuffle the batch; generate student rollouts and query $\pi_{T_k}$ at each domain-$k$ prefix
\State Compute the token-mean reverse-KL loss $\mathcal L_{\mathrm{OPD}}$; update exposures and KL statistics
\State Estimate remaining gaps $g_k$ and absorption rates $\widehat r_k$ using Eq.~\eqref{eq:gap}--\eqref{eq:rate}
\If{statistics are ready, $s\ge w$, and $s-s_{\mathrm{last}}\ge h$}
\State $\mathcal A\gets\{k:g_k>10^{-3},\ \widehat r_k>10^{-8},\ \widehat r_k\text{ is finite}\}$
\State Solve Eq.~\eqref{eq:bounded} for $p^*$ with the coverage fallbacks in Section~\ref{sec:bounded-sampling}
\State $p\gets\Pi_{\mathcal P}((1-\alpha)p+\alpha p^*)$; $s_{\mathrm{last}}\gets s$
    \EndIf
\State Update $\theta$ using $\mathcal L_{\mathrm{OPD}}$; save controller, sampler, and random states for recovery
\EndFor
\State \Return Student checkpoints and domain-scheduling traces
\end{algorithmic}
\end{algorithm*}

\section{Preliminaries}
\label{sec:prelim}
\subsection{MOPD Framework}
For each of $K$ domains, let $\mathcal D_k$ be its prompt distribution,
$\pi_{T_k}$ its frozen domain-expert teacher, and $\pi_\theta$ the shared
student. Domain sampling ratios $p\in\Delta_K$ satisfy $p_k\ge0$ and
$\sum_kp_k=1$. A batch of $B$ prompts assigns approximately $Bp_k$
examples to domain $k$; $e_k(t)$ counts cumulative prompt presentations, including repeated sampling.

\MOPD extends on-policy distillation to multiple teachers
\citep{ma2026mopd}. It samples $k\sim p$, $x\sim\mathcal D_k$, and a
student response $y\sim\pi_\theta(\cdot\mid x)$. At each student-visited
state $s_t=(x,y_{<t})$, the corresponding teacher supervises the student's
next-token distribution through reverse-KL
\citep{gu2024minillm,agarwal2024gkd}:
\begin{equation}
 d_{k,t}(\theta)
   =D_{\mathrm{KL}}\bigl(\pi_\theta(\cdot\mid s_t)
                        \|\pi_{T_k}(\cdot\mid s_t)\bigr).
\label{eq:mopd-prelim}
\end{equation}
For domain labels $k_i$ and non-padding response lengths $N_i$, the
batch objective is
\begin{equation}
 \mathcal L_{\mathrm{OPD}}(\theta)
   =\frac{\sum_{i=1}^B\sum_{t=1}^{N_i}d_{k_i,t}(\theta)}
          {\sum_{i=1}^B N_i}.
\label{eq:mopd-token-mean}
\end{equation}
Each divergence is evaluated on prompt $i$'s rollout, which is held
fixed during the update. Eq.~\eqref{eq:mopd-token-mean} averages
tokens; prompt shares need not equal token-loss shares.
Scheduling changes the prompt mixture, not teacher routing or loss weights.

\subsection{Domain Allocation Variants}
With $n_k$ usable training rows, Fixed retains
\begin{equation}
 p_k^{\mathrm{Fixed}}=\frac{n_k}{\sum_jn_j}.
\label{eq:fixed-mixture}
\end{equation}
These corpus-based ratios cannot respond to changing learning progress.

\Dthree adjusts ratios using normalized remaining reverse-KL
$\bar L_k$ and nonnegative descent velocity $v_k$, the recent relative
KL decrease \citep{sun2026d3mopd}. With temperature $T>0$, the original
method proposes
\begin{equation}
\begin{aligned}
 c_k&=\bar L_kv_k,\qquad z_k=c_k/\max_jc_j,\\
 \widetilde p_k^{\mathrm{D^3}}
   &=\frac{\exp(z_k/T)}{\sum_j\exp(z_j/T)}.
\end{aligned}
\label{eq:d3-mixture}
\end{equation}
If all $c_k=0$, the proposal is uniform; floors and caps
preserve coverage. This adapts to progress, but favoring large recent
decreases may undersample slowly improving domains.
We use the original \Dthree method; Appendix~\ref{app:comparator-details}
gives implementation details for its domain scheduling.

\section{Method}
\label{sec:method}
To address inflexible allocation and insufficient training for domains that
improve more slowly, \NashMOPD solves a local minimax game using existing
reverse-KL statistics. The equilibrium
gives more prompts to domains that need more examples for comparable
improvement. The controller maximizes the smallest predicted proportional
reduction to support bottleneck domains. Online updates adapt to
changing progress, while sampling ratios reflect each domain's
remaining training demand under the local model.
Bounds and smoothing preserve domain coverage and stable
allocations without extra teacher queries or rollouts.
Algorithm~\ref{alg:nash-controller} gives the complete training loop.

\subsection{Exposure-Normalized Progress Signals}
For each domain, the observed signal is mean per-token reverse-KL over
non-padding student response tokens. We clip negative empirical estimates
to zero and maintain an exponential moving average $L_k(t)$ with
coefficient $0.2$, initialized from the first valid observation.
The reference scale $L_k^{(0)}$ is the mean of the first five raw
nonnegative observations, floored at $10^{-8}$. With target ratio
$\rho=0.10$, the normalized remaining gap is
\begin{equation}
 g_k(t)=\left[
 \frac{L_k(t)-\rho L_k^{(0)}}{(1-\rho)L_k^{(0)}}
 \right]_+.
\label{eq:gap}
\end{equation}
This normalization accounts for differences in initial reverse-KL
scales. Gaps may exceed one when divergence increases; zero indicates
the KL target, not benchmark success. After the fifth observation, the
gap history starts at one. We retain up to 11 exposure--gap pairs and
use the oldest and newest records $[t_0,t_1]$, spanning up to ten observed
updates, to estimate
\begin{equation}
\begin{aligned}
 \widehat a_k(t)&=
 \frac{[g_k(t_0)-g_k(t_1)]_+}
      {\max(e_k(t_1)-e_k(t_0),1)},\\
 \widehat r_k(t)&=\frac{\widehat a_k(t)}
                         {\max(g_k(t_1),10^{-3})}.
\end{aligned}
\label{eq:rate}
\end{equation}
Here $\widehat a_k$ measures gap reduction per prompt, and
$\widehat r_k$ is the exposure-normalized absorption rate, or proportional
reduction per prompt. Dividing by exposure separates progress due to
receiving more prompts from progress per prompt; a zero estimate means
no positive decrease was measured in the window. These statistics are
collected from observations already available during ordinary distillation, without extra teacher queries
or rollouts.

\subsection{Local Minimax Allocation}
At one allocation update, fix positive gaps $g_k$ and finite positive
per-prompt reductions $a_k$; let $r_k=a_k/g_k$. The separable local model
predicts $g_k^+\approx g_k-Bp_ka_k$, with proportional reduction $Bp_kr_k$.
If these rates remained fixed, the predicted number of updates needed
to close every gap would be
\begin{equation}
 T_{\mathrm{loc}}(p)=\max_k\frac{g_k}{Bp_ka_k}
                   =\frac{1}{B\min_kp_kr_k}.
\label{eq:time}
\end{equation}
Minimizing Eq.~\eqref{eq:time} is thus equivalent to maximizing the
smallest predicted proportional reduction. The sampler chooses $p$,
while an adversary $q\in\Delta_K$ selects a bottleneck domain:
\begin{equation}
 u(p,q)=\sum_kq_kp_kr_k,
 \;\max_{p\in\Delta_K}\min_{q\in\Delta_K}u(p,q).
\label{eq:maximin}
\end{equation}
In Eq.~\eqref{eq:maximin}, minimizing over $q$ for fixed $p$
gives $\min_kp_kr_k$.
With no binding sampling bounds, the unique maximizing allocation is
\begin{equation}
 p_k^*=\frac{r_k^{-1}}{\sum_jr_j^{-1}}
       =\frac{g_k/a_k}{\sum_jg_j/a_j}.
\label{eq:equilibrium}
\end{equation}
To see why, an allocation with minimum payoff $m$ must satisfy
$p_k\ge m/r_k$. Summing yields $m\le(\sum_kr_k^{-1})^{-1}$;
Eq.~\eqref{eq:equilibrium} attains equality for every domain.
The equilibrium therefore assigns prompts in proportion to estimated
remaining training demand $g_k/a_k$. A smaller positive rate requires
more exposure to achieve comparable proportional improvement.
This result concerns a frozen-rate local model, not guaranteed
convergence of the nonlinear training trajectory. We re-estimate the
rates online; Appendix~\ref{app:theory} gives the saddle-point proof
and bounded-allocation analysis.

The adversary serves to identify the current bottleneck, rather than
introducing a trainable model. The controller computes the maximizing
allocation directly from domain statistics. Scheduling therefore fits
within the existing training loop, without alternating optimization
between players.

\subsection{Bounds, Smoothing, and Execution}
\label{sec:bounded-sampling}
Very small positive rates can produce extreme inverse-rate weights.
We constrain base sampling ratios to
$\mathcal P=\{p:\sum_kp_k=1,\ \ell\le p_k\le u\}$,
with $\ell=0.10$ and $u=0.70$. Feasibility requires $K\ell\le1\le Ku$.
Let $\mathcal A$ contain domains with $g_k>10^{-3}$ and finite
$\widehat r_k>10^{-8}$. Inactive domains initially receive the floor $\ell$.
On the active set, bisection selects a scalar $v$ such that
\begin{equation}
 \begin{aligned}
 p_k^*&=\operatorname{clip}(v/\widehat r_k,\ell,u),\\
 \sum_{k\in\mathcal A}p_k^*&=1-|\mathcal A^c|\ell.
 \end{aligned}
\label{eq:bounded}
\end{equation}
This maximizes the smallest active-domain payoff under the bounds;
it does not assign a positive modeled progress guarantee to zero-rate
domains. If active caps leave probability unallocated, active shares
are set to $u$ and the remainder goes to inactive domains. If none is
active, the target is uniform.
Training starts with Fixed's row-proportional ratios
(Eq.~\eqref{eq:fixed-mixture}). Observations
accumulate during a 20-update warm-up; scheduling starts only when all
domains have reference scales and at least two history records.
Every ten updates, we smooth the target with $\alpha=0.25$ and project
onto $\mathcal P$:
\begin{equation}
 p_t=\Pi_{\mathcal P}
       \big((1-\alpha)p_{t-1}+\alpha p_t^*\big).
\label{eq:smooth}
\end{equation}
The base mixture preserves domain coverage, while smoothing tempers
abrupt reallocations.
Each batch applies multiplicative jitter in $[0.7,1.3]$, renormalizes,
and uses largest-remainder rounding for $B$ prompt quotas. All schedules
share the permutation sampler. Bounds apply to post-warm-up base ratios;
jitter and rounding can change realized fractions. KL observations precede
the optimizer update, and revised ratios control subsequent batches.
Appendix~\ref{app:controller-state} gives sampler and checkpoint recovery details for resuming training.

\begin{table*}[t]
\centering
\BenchmarkTableStyle
\begin{tabular}{@{}L*{12}{B}@{}}
\toprule
\multicolumn{13}{c}{\textit{DeepSeek-R1-Distill-Qwen-1.5B}} \\
\midrule
 & \multicolumn{4}{c}{\textbf{Math}} & \multicolumn{4}{c}{\textbf{Code}} & \multicolumn{4}{c}{\textbf{Instruction following}} \\
\cmidrule(lr){2-5}\cmidrule(lr){6-9}\cmidrule(lr){10-13}
\multicolumn{1}{c}{Method} & \BenchmarkHeader{MATH}{avg@4} & \BenchmarkHeader{AIME}{avg@16} & \BenchmarkHeader{AMC}{avg@16} & Mean & \BenchmarkHeader{LCB}{avg@4} & \BenchmarkHeader{HE+}{avg@4} & \BenchmarkHeader{MB+}{avg@4} & Mean & \BenchmarkHeader{Multi-IF}{score} & \BenchmarkHeader{IF-L}{score} & \BenchmarkHeader{IF-S}{score} & Mean \\
\midrule
Initial student & 76.85 & 18.96 & 65.78 & 53.86 & 16.57 & 61.13 & 51.52 & 43.07 & 29.59 & 15.00 & 10.33 & 18.31 \\
\addlinespace[2pt]
Fixed & 83.05 & 26.04 & \textbf{79.69} & 62.93 & 21.57 & 68.14 & 56.68 & 48.80 & 40.25 & 16.67 & 12.33 & 23.08 \\
\Dthree & 83.25 & 26.04 & 79.06 & 62.78 & 19.29 & 66.31 & 57.01 & 47.54 & 41.54 & 18.00 & 13.67 & 24.40 \\
\NashMOPD & \textbf{85.95} & \textbf{27.50} & 78.91 & \textbf{64.12} & \textbf{23.29} & \textbf{68.75} & \textbf{57.47} & \textbf{49.84} & \textbf{43.01} & \textbf{20.00} & \textbf{14.67} & \textbf{25.89} \\
\bottomrule
\end{tabular}
\captionsetup{labelsep=period}
\caption{Benchmark scores in \%. \Dthree denotes the original method described in the experimental setup. MATH, AIME, AMC, and LCB denote MATH-500, AIME25, AMC23, and LiveCodeBench v6. HE+ and MB+ denote HumanEval+ and MBPP+. IF-L and IF-S denote IFBench loose and strict scores. \textit{Mean} is the unweighted mean of the three displayed metrics within each domain. Bold marks the best method, including ties.}
\label{tab:deepseek-results}
\end{table*}

\begin{table*}[t]
\centering
\BenchmarkTableStyle
\begin{tabular}{@{}L*{8}{B}@{}}
\toprule
\multicolumn{9}{c}{\textit{Qwen3-8B-Base}} \\
\midrule
 & \multicolumn{4}{c}{\textbf{Math}} & \multicolumn{4}{c}{\textbf{Code}} \\
\cmidrule(lr){2-5}\cmidrule(lr){6-9}
\multicolumn{1}{c}{Method} & \BenchmarkHeader{MATH}{avg@4} & \BenchmarkHeader{AIME}{avg@16} & \BenchmarkHeader{AMC}{avg@16} & Mean & \BenchmarkHeader{LCB}{avg@4} & \BenchmarkHeader{HE+}{avg@4} & \BenchmarkHeader{MB+}{avg@4} & Mean \\
\midrule
Initial student & 64.25 & 9.38 & 55.00 & 42.88 & 16.57 & 69.82 & 43.39 & 43.26 \\
\addlinespace[2pt]
Fixed & 86.70 & 22.71 & 78.28 & 62.56 & 26.86 & 83.08 & 60.19 & 56.71 \\
\Dthree & \textbf{87.40} & 21.67 & 79.84 & 62.97 & 27.43 & 82.62 & 61.97 & 57.34 \\
\NashMOPD & \textbf{87.40} & \textbf{23.13} & \textbf{80.16} & \textbf{63.56} & \textbf{28.57} & \textbf{83.23} & \textbf{62.37} & \textbf{58.06} \\
\bottomrule
\end{tabular}
\captionsetup{labelsep=period}
\caption{Benchmark scores in \%. MATH, AIME, AMC, and LCB denote MATH-500, AIME25, AMC23, and LiveCodeBench v6. HE+ and MB+ denote HumanEval+ and MBPP+. \textit{Mean} is the unweighted mean of the three displayed metrics within each domain. Bold marks the best method, including ties.}
\label{tab:qwen3-results}
\end{table*}

\paragraph{Measuring joint convergence.}
\label{sec:joint-convergence}
Controller feedback uses reverse-KL; joint target attainment uses
held-out domain scores $S_{k,t}$ and thresholds $\tau_k$.
For the measured checkpoint set $\mathcal C$, the first checkpoint meeting all domain targets is
\begin{equation}
 t_{\mathrm{first}}=\min\{t\in\mathcal C:
                    S_{k,t}\ge\tau_k\ \text{for all }k\}.
\label{eq:joint-crossing}
\end{equation}
The first sustained crossing additionally requires every later measured
checkpoint through the endpoint to satisfy all thresholds. We report
both alongside endpoint accuracy.

\section{Experiments}
\label{sec:experiments}

\subsection{Experimental Setup}
\paragraph{Models and training data.}
DeepSeek-R1-Distill-Qwen-1.5B \citealp{deepseek2025r1} trains on
DAPO mathematics \citealp{yu2025dapo},
TACO code \citealp{li2023taco}, and Multi-IF
conversations \citealp{he2024multiif}. Its teachers are
JustRL-DeepSeek-1.5B \citealp{he2025justrl},
DeepCoder-1.5B-Preview, and UltraData-IF-1.5B \citealp{fu2026oneexampleopd},
respectively. The Multi-IF training split excludes 256 held-out conversations.
Qwen3-8B-Base \citealp{yang2025qwen3} uses the same math
and code corpora. Its teachers are two Qwen3-8B-Base versions trained
via reinforcement learning on math and code, respectively.
This tests allocation across scales and domain counts.

\paragraph{Baselines and training steps.}
We compare Fixed, the original \Dthree baseline, and \NashMOPD through
step 200, evaluating checkpoints every 20 optimizer updates.
Final benchmark scores use the step-200 checkpoint.
Fixed retains row-proportional base ratios throughout training. \Dthree combines
normalized reverse-KL with descent velocity. Its original scheduling
procedure is specified in Appendix~\ref{app:comparator-details}.
\NashMOPD uses the controller in Section~\ref{sec:method}.
The baselines differ in allocation, not in the distillation objective.
Their common training configuration is listed in
Appendix~\ref{app:training-settings}.
The shared budget isolates allocation effects.

\paragraph{Benchmarks and metrics.}
Endpoint accuracy is avg@4 on MATH-500 \citealp{hendrycks2021math},
LiveCodeBench v6, or LCB, \citealp{jain2025livecodebench}, HumanEval+, and
MBPP+ \citealp{liu2023evalplus}, and avg@16 on AIME25 and AMC23.
avg@$k$ averages correctness over $k$ candidates per question, whereas
pass@$k$ measures the fraction of questions solved by at least one of
those candidates. We report both.
Instruction following uses Multi-IF's three-turn score and IFBench's
loose and strict prompt accuracy \citealp{pyatkin2025ifbench}.
Checkpoint curves use MATH-500 and LCB pass@1, with Multi-IF additionally
included in the three-domain experiment.
This separates average accuracy from candidate coverage.

\subsection{Endpoint Accuracy}
\label{sec:primary-results}
Tables~\ref{tab:deepseek-results} and~\ref{tab:qwen3-results} report
endpoint accuracy. \NashMOPD leads eight of nine individual benchmark
metrics on the DeepSeek-R1-Distill-Qwen-1.5B student and is best or tied best on all six on the Qwen3-8B-Base
student. Domain means are reported separately from these counts.

\paragraph{Three-domain experiment.}
On the DeepSeek-R1-Distill-Qwen-1.5B student, \NashMOPD scores 85.95\% on MATH-500 avg@4,
23.29\% on LCB avg@4, and 43.01\% on Multi-IF. These exceed Fixed by
2.90, 1.72, and 2.76 percentage points, and \Dthree by 2.70, 4.00, and
1.47 points, respectively. Fixed has the highest AMC23 avg@16 score
at 79.69\%, compared with 78.91\% for \NashMOPD.
Relative to the initial student's 76.85\% and the math teacher's 86.10\%
MATH-500 avg@4 scores, \NashMOPD achieves a gap recovery ratio of 98.4\%
\citealp{fu2026oneexampleopd}, relative to the 9.25-point initial
teacher--student performance gap, compared with 67.0\% for Fixed.
The remaining math gap is only 0.15 points, versus 3.05 points for Fixed at the final checkpoint.

\paragraph{Two-domain experiment.}
On the Qwen3-8B-Base student, \NashMOPD scores 87.40\% on MATH-500 avg@4 and
28.57\% on LCB avg@4, exceeding Fixed by 0.70 and 1.71 points.
It ties \Dthree on MATH-500 and exceeds its LCB score by 1.14 points.
Mathematics and code means reach 63.56\% and 58.06\%, exceeding Fixed
by 1.00 and 1.35 points, respectively. Both means also exceed
\Dthree's by 0.59 and 0.72 points, so the larger code improvement
is accompanied by a gain in mathematics on the larger base student.

\begin{figure*}[t]
\begin{minipage}{\textwidth}
\captionsetup{hypcap=true,type=table}
\begingroup
\centering
\BenchmarkTableStyle
\begin{tabular}{@{}L*{8}{B}@{}}
\toprule
 & \multicolumn{4}{c}{\textbf{Math}} & \multicolumn{4}{c}{\textbf{Code}} \\
\cmidrule(lr){2-5}\cmidrule(lr){6-9}
\multicolumn{1}{c}{Method} & \BenchmarkHeader{MATH}{pass@4} & \BenchmarkHeader{AIME}{pass@16} & \BenchmarkHeader{AMC}{pass@16} & Mean & \BenchmarkHeader{LCB}{pass@4} & \BenchmarkHeader{HE+}{pass@4} & \BenchmarkHeader{MB+}{pass@4} & Mean \\
\midrule
\multicolumn{9}{c}{\textit{DeepSeek-R1-Distill-Qwen-1.5B}} \\
\midrule
Initial student & 88.00 & 43.33 & 92.50 & 74.61 & 19.43 & 78.66 & 67.20 & 55.10 \\
\addlinespace[2pt]
Fixed & 92.80 & 36.67 & \textbf{95.00} & 74.82 & 24.00 & 81.10 & 67.72 & 57.61 \\
\Dthree & 91.40 & 50.00 & \textbf{95.00} & 78.80 & 21.71 & 80.49 & 70.11 & 57.44 \\
\NashMOPD & \textbf{93.60} & \textbf{53.33} & \textbf{95.00} & \textbf{80.64} & \textbf{25.14} & \textbf{82.32} & \textbf{70.37} & \textbf{59.28} \\
\midrule
\multicolumn{9}{c}{\textit{Qwen3-8B-Base}} \\
\midrule
Initial student & 86.60 & 33.33 & 87.50 & 69.14 & 19.43 & 87.20 & 75.40 & 60.68 \\
\addlinespace[2pt]
Fixed & 93.80 & 43.33 & 92.50 & 76.54 & 29.71 & 90.24 & 80.69 & 66.88 \\
\Dthree & \textbf{94.00} & 36.67 & 95.00 & 75.22 & 30.86 & \textbf{92.07} & 79.37 & 67.43 \\
\NashMOPD & \textbf{94.00} & \textbf{46.67} & \textbf{97.50} & \textbf{79.39} & \textbf{32.00} & 91.46 & \textbf{80.95} & \textbf{68.14} \\
\bottomrule
\end{tabular}
\captionsetup{labelsep=period}
\caption{pass@$k$ coverage in \%. Benchmark abbreviations follow
Tables~\ref{tab:deepseek-results} and~\ref{tab:qwen3-results}.
\textit{Mean} averages the three benchmarks within each domain.
Bold marks the best method within each student, including ties.}
\label{tab:coverage-results}

\endgroup
\end{minipage}
\par\vspace{4pt}
\begin{minipage}{\textwidth}
\centering
\captionsetup{hypcap=true,type=figure,labelsep=period}
\includegraphics[width=0.75\textwidth]{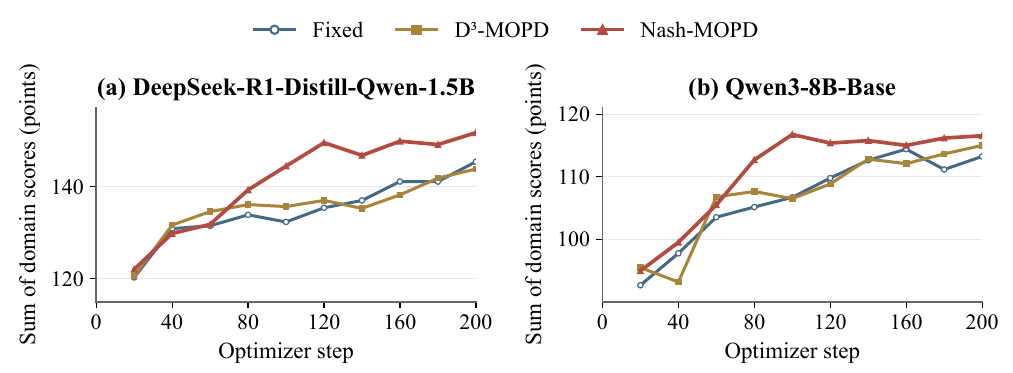}
\caption{Held-out domain-score totals for DeepSeek-R1-Distill-Qwen-1.5B
(left: math, code, and Multi-IF) and Qwen3-8B-Base (right: math and code).
Math/code use pass@1. Totals are compared within each student.}
\label{fig:score-sums}
\end{minipage}
\vspace{-12pt}
\end{figure*}

\begin{figure*}[t]
\centering
\captionsetup{labelsep=period}
\includegraphics[width=\textwidth]{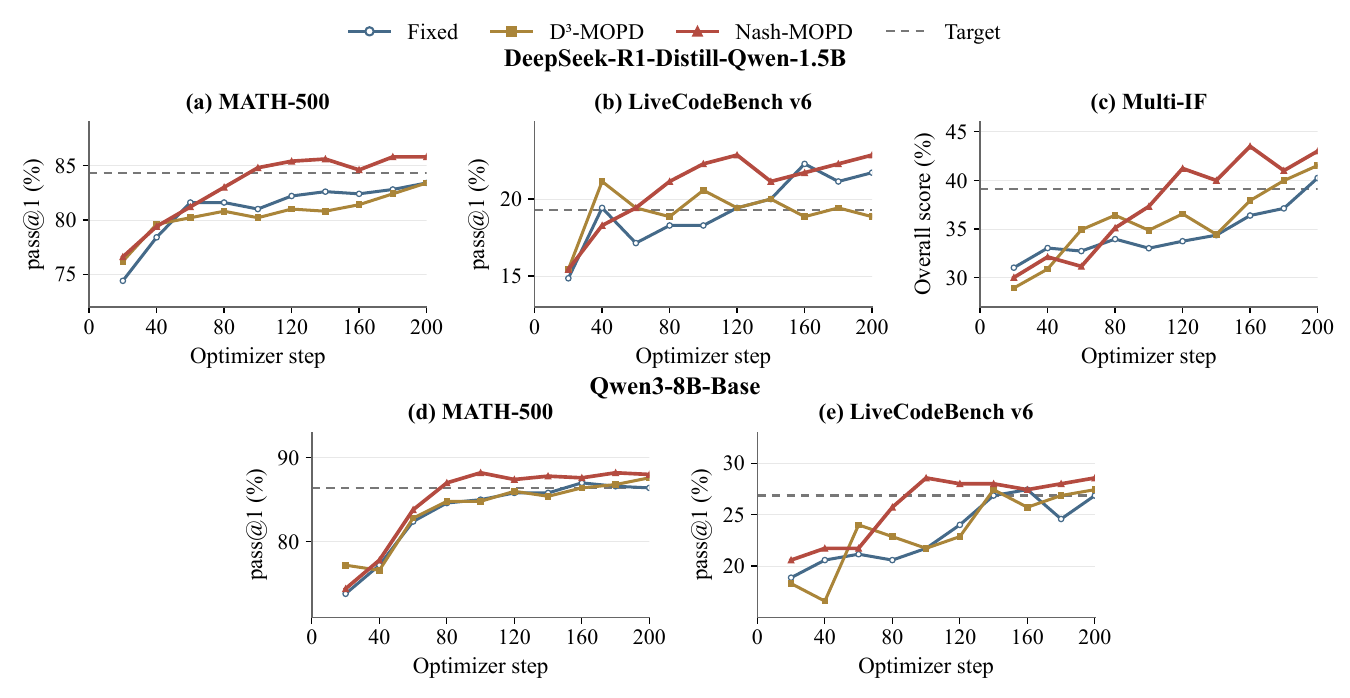}
\caption{Checkpoint scores for DeepSeek-R1-Distill-Qwen-1.5B (top)
and Qwen3-8B-Base (bottom). Dashed lines mark teacher-referenced
targets (top) and Fixed's final math/code scores (bottom).}
\label{fig:checkpoint-curves}
\label{fig:primary-curves}
\label{fig:qwen3-8b-base-curves}
\vspace{-12pt}
\end{figure*}

\subsection{Convergence Speed}
\label{sec:efficiency}
We measure convergence speed by the optimizer updates needed to match
baseline endpoint mathematics and code performance.
Figures~\ref{fig:score-sums} and~\ref{fig:checkpoint-curves} report domain-score totals
and checkpoint curves at 20-step intervals. Math and code use pass@1,
whereas the endpoint tables use multi-candidate average accuracy.

\paragraph{Three-domain experiment.}
On the DeepSeek-R1-Distill-Qwen-1.5B student, \NashMOPD's step-100
math/code scores of 84.80/22.29\% already exceed Fixed's step-200
scores of 83.40/21.71\% and \Dthree's 83.40/18.86\%.
It therefore surpasses both baselines' final mathematics and code
performance with half the training updates. It also reaches all three
teacher-referenced targets of 84.3/19.3/39.1\% on math/code/Multi-IF
at step 120, only 20 updates after exceeding the baseline math/code
endpoints, and sustains them at later measured checkpoints.
Neither baseline reaches all three targets within 200 optimization updates.

\paragraph{Two-domain experiment.}
On the Qwen3-8B-Base student, \NashMOPD reaches Fixed's final
math/code scores of 86.40/26.86\% at step 100, scoring 88.20/28.57\%, with
margins of 1.80 and 1.71 points above these thresholds.
Fixed and \Dthree first reach the same reference level at steps 160
and 180, respectively. \NashMOPD thus uses 60 and 80 fewer updates,
reductions of 37.5\% and 44.4\%, and remains above both thresholds
through step 200.
Figure~\ref{fig:score-sums} also shows that \NashMOPD has the highest
domain-score totals from step 100 onward for both students.

\paragraph{Training cost.}
Figure~\ref{fig:training-time} reports full 200-step training time.
Fixed, \Dthree, and \NashMOPD take 6.63/6.49/6.81 hours on the DeepSeek-R1-Distill-Qwen-1.5B
student and 12.97/12.97/13.21 hours on the Qwen3-8B-Base student, respectively.
These totals exceed Fixed's by 0.18 and 0.24 hours
or 2.7\% and 1.9\%, respectively.
We measure convergence speed in optimizer updates.
Wall-clock time to the earlier performance crossings was not measured.

\begin{figure}[!t]
\centering
\captionsetup{labelsep=period}
\includegraphics[width=\linewidth]{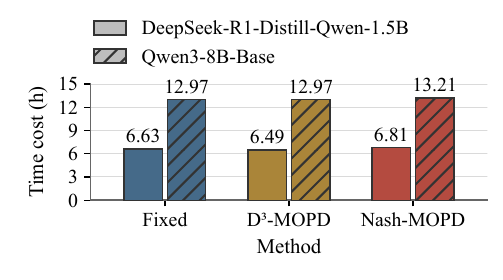}
\caption{200-step training time without held-out evaluation:
DeepSeek-R1-Distill-Qwen-1.5B (solid) and Qwen3-8B-Base (hatched).}
\label{fig:training-time}
\end{figure}

\subsection{Additional Benchmarks and Candidate Coverage}
\label{sec:generalization}
AIME25 and AMC23 use sixteen candidates per question. HumanEval+ and MBPP+
{\setcitestyle{citesep={,}}\citealp{chen2021evaluatingcode,austin2021programsynthesis,liu2023evalplus}}
use four. IFBench \citealp{pyatkin2025ifbench} reports loose and strict
accuracy from one response per prompt.

\paragraph{Additional benchmark accuracy.}
For the DeepSeek-R1-Distill-Qwen-1.5B student, \NashMOPD has the highest AIME25 avg@16 at 27.50\%,
HumanEval+ avg@4 at 68.75\%, and MBPP+ avg@4 at 57.47\%.
Its IFBench loose and strict scores are 20.00\% and 14.67\%, exceeding
Fixed by 3.33 and 2.34 points, roughly 20\% relative gains under both
scoring criteria. For the Qwen3-8B-Base student, \NashMOPD leads
AIME25 avg@16 at 23.13\%, AMC23 avg@16 at 80.16\%, HumanEval+ avg@4
at 83.23\%, and MBPP+ avg@4 at 62.37\%.

\paragraph{Coverage under multiple candidates.}
Table~\ref{tab:coverage-results}
reports pass@4 for MATH-500,
LiveCodeBench v6, HumanEval+, and MBPP+, and pass@16 for AIME25 and AMC23.
\NashMOPD is best or tied best on all six coverage metrics for the DeepSeek-R1-Distill-Qwen-1.5B student
and five of six for the Qwen3-8B-Base student.
On the DeepSeek-R1-Distill-Qwen-1.5B student, \NashMOPD's AIME25 coverage reaches 53.33\%, compared
with 36.67\% for Fixed, a gain of about 16.7 points.
MATH-500 and LCB coverage rise to 93.60\%
and 25.14\%. On the Qwen3-8B-Base student, \NashMOPD leads AIME25 at 46.67\%, AMC23
at 97.50\%, LCB at 32.00\%, and MBPP+ at 80.95\%, and ties the best
MATH-500 coverage at 94.00\%. On HumanEval+ for the Qwen3-8B-Base student, \Dthree has higher pass@4 coverage
than \NashMOPD at 92.07\% versus 91.46\%, despite lower avg@4 accuracy
at 82.62\% versus 83.23\%. Here, \NashMOPD's 0.61-point average-accuracy
gain coexists with a 0.61-point coverage shortfall, illustrating that
the two metrics can rank methods differently.

\FloatBarrier

\section{Conclusion}
\NashMOPD allocates prompts by maximizing the smallest predicted
absorption rate of the remaining teacher--student reverse-KL gaps. Under a separable, frozen-rate
progress model, the unconstrained equilibrium follows an inverse-rate
rule. Online rate estimates, bounded shares, and smoothing implement this
allocation during training.
The DeepSeek-R1-Distill-Qwen-1.5B student reaches and sustains all three teacher-referenced targets
from step 120, while the baselines remain short of the joint target through
step 200. The Qwen3-8B-Base student reaches Fixed's final math/code performance at
step 100, requiring 37.5\% and 44.4\% fewer updates than Fixed and
\Dthree for first attainment. At the endpoint, \NashMOPD leads eight of
nine metrics on the DeepSeek-R1-Distill-Qwen-1.5B student and is best or tied best on all
six metrics on the Qwen3-8B-Base student, with no extra teacher queries or rollouts.

\clearpage
\raggedbottom
\section*{Limitations}
\NashMOPD relies on local, separable progress estimates. Future work
could model cross-domain transfer and forgetting, quantify rate
uncertainty, and handle delayed feedback. Incorporating token counts
and training time could improve resource allocation. Learning sampling
bounds, update intervals, and smoothing could reduce manual tuning.
More domains, model sizes, and training budgets would test these extensions.

\clearpage
\flushbottom
\appendix
\section{Theory of the Allocation Game}
\label{app:theory}

\subsection{Local surrogate and the zero-sum game}
At one scheduler update, fix normalized gaps $g_k>0$ and estimated
per-prompt reductions $a_k>0$ for the active domains. A batch of $B$
prompts with allocation $p\in\Delta_K$ predicts gap reduction $Bp_ka_k$
and fractional progress $Bp_kr_k$, where $r_k=a_k/g_k$.
Ignoring truncation at the target, define the local time surrogate
\begin{equation}
  T_{\mathrm{loc}}(p)=\max_{k\in[K]}\frac{1}{B p_k r_k}
  =\frac{1}{B\min_k p_k r_k}.
  \label{eq:app-local-time}
\end{equation}
This is a local planning quantity, not a guarantee for nonlinear training:
student parameters, responses, losses, and rates can change after each update.

The sampler is the maximizing player and the adversary is a distribution
$q\in\Delta_K$ over domains. Their zero-sum payoff is
\begin{equation}
  u(p,q)=\sum_{k=1}^K q_k p_k r_k.
  \label{eq:app-payoff}
\end{equation}
For fixed $p$, the adversary places all mass on an index attaining
$\min_k p_k r_k$, so $\min_{q\in\Delta_K}u(p,q)=\min_kp_kr_k$. The sampler's
problem is the linear maximin program
\begin{equation}
  \nu^*=\max_{p\in\Delta_K}\min_{k\in[K]}p_k r_k.
  \label{eq:app-primal}
\end{equation}

\begin{theorem}[Interior maximin solution]
\label{thm:interior}
For a fixed active set and finite rates $r_k>0$, with no active lower or upper
share constraints, the unique maximizing allocation in
Eq.~\eqref{eq:app-primal} is
\begin{equation}
  p_k^*=\frac{r_k^{-1}}{\sum_{j=1}^K r_j^{-1}},\qquad
  \nu^*=\left(\sum_{j=1}^K r_j^{-1}\right)^{-1}.
  \label{eq:app-interior-solution}
\end{equation}
The associated worst-domain local time is
$T_{\mathrm{loc}}(p^*)=B^{-1}\sum_jr_j^{-1}$. The pair
$(p^*,q^*)$, where $q_k^*=p_k^*$, is a saddle point of
Eq.~\eqref{eq:app-payoff}.
\end{theorem}

\begin{proof}
Let $p$ be any feasible allocation and write
$m(p)=\min_kp_kr_k$. Then $p_k\ge m(p)/r_k$ for every $k$, so
\begin{equation}
\begin{aligned}
  1&=\sum_kp_k\ge m(p)\sum_k r_k^{-1},\\
  m(p)&\le\left(\sum_k r_k^{-1}\right)^{-1}.
\end{aligned}
  \label{eq:app-upper-bound}
\end{equation}
The proposed allocation sums to one and satisfies
$p_k^*r_k=(\sum_jr_j^{-1})^{-1}$ for every domain, attaining this upper
bound. Equality in Eq.~\eqref{eq:app-upper-bound} requires
$p_k=m(p)/r_k$ for every $k$, which proves uniqueness. Substituting
$\nu^*$ into Eq.~\eqref{eq:app-local-time} gives the stated local time.
For the saddle-point claim, $p_k^*r_k=\nu^*$ makes
$u(p^*,q)=\nu^*$ for every $q$. Also, since $q_k^*r_k=\nu^*$,
$u(p,q^*)=\nu^*$ for every $p$. Thus neither player can improve its
payoff by a unilateral change.
\end{proof}

The minimum payoff requires fractional progress in every domain.
Theorem~\ref{thm:interior} therefore also minimizes worst-domain local time
under frozen rates; the online controller re-estimates these rates
after subsequent updates.

\subsection{Bounded shares and water-filling}
The implementation imposes a floor $\ell$ and cap $u$ on domain sampling ratios. We
assume $K\ell\le1\le Ku$. Let $\mathcal{A}$ contain domains with a gap above
the target threshold and a positive finite estimated rate; let $\mathcal{I}$
be the other domains. The code fixes $p_i=\ell$ for $i\in\mathcal{I}$ and
maximizes the worst rate over $\mathcal{A}$ using the remaining mass
$b=1-|\mathcal{I}|\ell$. This convention is part of the controller: a
zero-rate domain is not assigned a modeled positive payoff, so its share is a
per-domain floor rather than a maximin guarantee.

For $|\mathcal{A}|>0$, define
\begin{equation}
  \phi(v)=\sum_{k\in\mathcal{A}}
     \operatorname{clip}\!\left(\frac{v}{r_k},\ell,u\right).
  \label{eq:app-waterfill}
\end{equation}
When $b\le |\mathcal{A}|u$, the bisection in the implementation finds a
level $v$ for which $\phi(v)=b$ (up to numerical tolerance), and sets the
positive-rate shares to the clipped values in Eq.~\eqref{eq:app-waterfill}.
If $b>|\mathcal{A}|u$, all positive-rate domains receive $u$ and the forced
remainder is shared among inactive domains. The following result states what
this construction does and, just as importantly, what its scalar $v$ does
not mean when a cap binds.

\begin{theorem}[Optimality of the bounded allocation]
\label{thm:bounded}
Suppose $|\mathcal{A}|>0$, $r_k>0$ for all $k\in\mathcal{A}$, and the
floor/cap constraints are feasible. If $b\le |\mathcal{A}|u$, any vector
$p_k=\operatorname{clip}(v/r_k,\ell,u)$ satisfying
$\sum_{k\in\mathcal{A}}p_k=b$ maximizes
$\min_{k\in\mathcal{A}}p_kr_k$ over
$\ell\le p_k\le u$ with total mass $b$. If $b>|\mathcal{A}|u$, setting
$p_k=u$ for all $k\in\mathcal{A}$ maximizes the same objective; any feasible
distribution of the remaining probability over $\mathcal{I}$ leaves that
active-domain optimum unchanged. The maximizing allocation need not be unique.
\end{theorem}

\begin{proof}
Write $m(p)=\min_{k\in\mathcal{A}}p_kr_k$. First consider a solution
of Eq.~\eqref{eq:app-waterfill}. If an upper-clipped domain has payoff
below $v$, choose $j$ with the smallest payoff $u r_j$ among such domains.
Every feasible allocation then has $m(p)\le u r_j$. The clipped vector
assigns $p_j=u$; every uncapped domain has payoff $v$, every floor-clipped domain
has payoff at least $v$, and every other upper-clipped domain has payoff at
least $u r_j$. Its minimum therefore equals this upper bound, so no
allocation can do better.
If no upper-clipped domain has payoff below $v$, all domain payoffs under the
clipped vector are at least $v$. Any allocation with minimum payoff strictly
above $v$ must assign more than $v/r_k$ to every domain with
$\ell\le v/r_k<u$, including a domain exactly at its floor. Domains with
$v/r_k<\ell$ can remain at their floor and already have payoff above $v$. If
some domain has $v/r_k=u$, strict improvement is impossible for that domain
because of its cap. Otherwise, when at least one domain needs a strict increase,
the required total exceeds the clipped vector's mass $b$. If all domains have
$v/r_k<\ell$, feasibility fixes every share at its floor. Thus no feasible
vector has a larger minimum payoff. Finally, when $b>|\mathcal{A}|u$, every active coordinate
is already at its maximum, so no feasible allocation can increase any
positive-rate domain's share; the minimum active payoff is therefore maximal.
The inactive residual is feasible because $K\ell\le1\le Ku$.
\end{proof}

The bounded objective has the same domain-adversary interpretation: for
fixed shares, minimizing the bilinear payoff over distributions on
$\mathcal A$ selects an active bottleneck. The implementation needs only
the maximizing allocation from Theorem~\ref{thm:bounded}; it does not
compute a dual distribution or alternate between game players.

The level $v$ equalizes payoffs only for interior coordinates. A capped
domain can have payoff $ur_k<v$ and become the bottleneck; plateaus of
$\phi$ can make $v$ non-unique. Hence $v$ should not be interpreted as
an empirical Nash value when constraints bind.

For a feasible active budget, the scalar search is well-defined:
$\phi$ is continuous and nondecreasing, with
$\phi(0)=|\mathcal A|\ell$ and
$\phi(u\max_{k\in\mathcal A}r_k)=|\mathcal A|u$.
The implementation uses 80 bisection iterations within this bracket.
Share clipping and a final projection handle finite precision; the root
level describes the bounded allocation, not a measured convergence rate.

\subsection{Sensitivity to estimated positive rates}
The controller uses window estimates $\widehat r_k$. The following result
quantifies how a specified error changes its local planning objective.
It compares the same active domain set and feasible share constraints;
changes in which domains pass the positive-rate test are addressed separately
in Section~\ref{app:zero-rates}.

\begin{proposition}[Deterministic rate sensitivity]
\label{thm:additive-error}
Let $\mathcal A$ be a fixed nonempty active set and let
$\mathcal P\subseteq\Delta_K$ be a nonempty compact feasible allocation set
with the controller's coverage convention held fixed. Suppose
$r_k,\widehat r_k>0$ and
$\max_{k\in\mathcal A}|r_k-\widehat r_k|\le\epsilon$.
If $\widehat p$ maximizes $\min_{k\in\mathcal A}p_k\widehat r_k$ on
$\mathcal P$, then
\begin{equation}
\begin{aligned}
 \min_{k\in\mathcal A}\widehat p_kr_k
 &\ge \nu_{\mathcal A}(r)-2\epsilon,\\
 \nu_{\mathcal A}(r)
 &=\max_{p\in\mathcal P}\min_{k\in\mathcal A}p_kr_k.
\end{aligned}
\label{eq:app-additive-error}
\end{equation}
\end{proposition}

\begin{proof}
For every feasible $p$, $0\le p_k\le1$ bounds the change in each active
product by $\epsilon$. Taking the minimum preserves that bound, so
$|\min_{k\in\mathcal A}p_kr_k-
  \min_{k\in\mathcal A}p_k\widehat r_k|\le\epsilon$.
Maximizing over the same compact set also gives
$|\nu_{\mathcal A}(r)-\nu_{\mathcal A}(\widehat r)|\le\epsilon$.
The estimated optimizer therefore satisfies
\begin{equation*}
 \min_{k\in\mathcal A}\widehat p_kr_k
 \ge\nu_{\mathcal A}(\widehat r)-\epsilon
 \ge\nu_{\mathcal A}(r)-2\epsilon.
\end{equation*}
\end{proof}

This result applies to the bounded target before temporal smoothing. It is
conditional on a specified perturbation of the local rates. Our moving-window
EMA estimator is temporally dependent, and the experiments do not supply a
calibrated bound on $\epsilon$. Consequently the proposition explains
sensitivity of the planning objective without assigning a confidence level
or certifying progress for the reported training runs.

\subsection{Smoothing, batch-level jitter, and finite batches}
Let $q_t$ be the bounded target, $p_t$ the smoothed allocation, and
$\alpha\in(0,1]$ the smoothing coefficient. Convex combinations of feasible
shares remain feasible, so Euclidean projection is a numerical safeguard
after entry into the capped simplex. It also enforces bounds at the first
dynamic update, since the instruction-following warm-up share is below
$0.1$. The lemma assumes a feasible initial allocation.

\begin{lemma}[Smoothing contraction and tracking]
\label{lem:smoothing}
For a fixed target $q$, the update
$p_t=\Pi_{\mathcal{P}}((1-\alpha)p_{t-1}+\alpha q)$ satisfies
$\|p_t-q\|_2\le(1-\alpha)\|p_{t-1}-q\|_2$. If the targets vary with
$\|q_t-q_{t-1}\|_2\le\xi$ and the update uses $q_t$, then
\begin{equation}
  \|p_t-q_t\|_2\le(1-\alpha)\|p_{t-1}-q_{t-1}\|_2
                    +(1-\alpha)\xi.
  \label{eq:app-tracking}
\end{equation}
Consequently $\limsup_t\|p_t-q_t\|_2$ is at most
$(1-\alpha)\xi/\alpha$ when the stated variation bound holds.
\end{lemma}

\begin{proof}
Euclidean projection onto a closed convex set is non-expansive and fixes
$q$. Therefore
$\|p_t-q\|_2\le\|(1-\alpha)p_{t-1}+\alpha q-q\|_2$,
which gives the contraction. For a varying target, compare the projected
update with $q_t$ and apply the triangle inequality:
\begin{align*}
\|p_t-q_t\|_2
&\le(1-\alpha)\|p_{t-1}-q_t\|_2\\
&\le(1-\alpha)\|p_{t-1}-q_{t-1}\|_2\\
&\qquad +(1-\alpha)\xi.
\end{align*}
Iterating this scalar recursion gives a geometric transient plus a geometric
sum bounded by $(1-\alpha)\xi/\alpha$.
\end{proof}

The lemma concerns the tracking of target shares. It supplies no contraction
of the student's KL gap or benchmark error. The runs use $\alpha=0.25$;
we do not estimate a stationary target-variation bound from their traces.

For one batch, the implementation multiplies each share $p_k$ by
$1+\eta_k$, with $|\eta_k|\le\gamma$, and renormalizes. Let $\widetilde p$
be this jittered distribution. Since
$1-\gamma\le\sum_jp_j(1+\eta_j)\le1+\gamma$,
\begin{equation}
  \frac{1-\gamma}{1+\gamma}p_k\le\widetilde p_k\le
  \frac{1+\gamma}{1-\gamma}p_k.
  \label{eq:app-jitter}
\end{equation}
For a batch of $B$ prompts, largest-remainder apportionment first assigns
$\lfloor B\widetilde p_k\rfloor$ examples to each domain and then distributes
the remaining prompts according to fractional remainders. If $n_k$ is the
resulting integer quota, then
\begin{equation}
  \sum_kn_k=B,\qquad
  \left|\frac{n_k}{B}-\widetilde p_k\right|<\frac1B.
  \label{eq:app-rounding}
\end{equation}
The bound is deterministic and concerns batch composition only. It does not
bound gradient noise or model quality. The random perturbations are sampled
independently per batch. The sampler is shared by Fixed, \Dthree, and \NashMOPD
in both the DeepSeek-R1-Distill-Qwen-1.5B and Qwen3-8B-Base experiments.

\subsection{Zero-rate domains and the scope of the guarantee}
\label{app:zero-rates}
The bounded solver marks domain $k$ as progressing only if
$g_k>10^{-3}$, $\widehat r_k>10^{-8}$, and its rate is finite. A domain outside
this set receives its per-domain floor when a feasible positive-rate allocation
can consume the remaining probability. If upper bounds prevent that allocation
from consuming all mass, the residual is split equally over the other domains.
When no domain passes the progressing-domain test, the target is uniform. These
rules make every sampler update defined without dividing by zero.

They also change the scope of the optimization. If an unfinished domain has
$r_k=0$, then $\min_k p_kr_k=0$ for every feasible $p$ in the full-domain game,
and its modeled completion time is infinite. No allocation can certify
positive all-domain progress from those rates. The controller instead optimizes
on the positive-rate subset while retaining exposure to the other domains. A
zero estimate can represent a noisy window, a plateau, interference, or an
inadequate teacher; the observations do not identify which explanation holds.
In particular, allocating the minimum share to a zero-rate domain should not be
interpreted as the best recovery policy for a stalled domain. 

The diagonal local model additionally assumes that a domain's improvement is
attributable to its own prompt exposure. A shared student can improve one domain
from another domain's data or suffer negative transfer. The present analysis does not
model those interactions, prove monotone KL descent, or establish a
connection from KL descent to a benchmark threshold.
The theoretical statements therefore concern allocation under specified local
rates, whereas the results test whether that controller is useful in the
observed nonlinear training system.

\clearpage

\twocolumn[{
\begin{minipage}{\textwidth}
\nolinenumbers
\section{Implementation Details}
\label{app:implementation}
\label{app:eval-details}
This appendix complements Algorithm~\ref{alg:nash-controller} with the
controller state, comparator settings, and shared training settings.

\par\vspace{6pt}
\begin{minipage}{\textwidth}
\makeatletter\def\@captype{table}\makeatother
\captionsetup{hypcap=false}
\nolinenumbers
\centering\small
\setlength{\tabcolsep}{5pt}
\begingroup
\newcommand{\AppendixSettingsTable}{\begin{tabular*}{\linewidth}{@{\extracolsep{\fill}}lll@{}}
\toprule
Setting & Value & Applied to \\
\midrule
Global prompt batch / responses per prompt & 64 / 1 & All training runs \\
Maximum prompt / response tokens & 4,096 / 7,168 & All training runs \\
Student optimizer / learning rate / weight decay & AdamW / $10^{-6}$ / 0.01 & All training runs \\
Training rollout / teacher temperature & 1.0 / 1.0 & All training runs \\
Loss aggregation / teacher precision & Token mean / bfloat16 & All training runs \\
KL scale observations / EMA coefficient & 5 / 0.2 & Domain-sampler controllers \\
KL target ratio / remaining-gap floor & 0.10 / $10^{-3}$ & \NashMOPD signals \\
Positive-rate cutoff & $10^{-8}$ & \NashMOPD allocation \\
Per-domain floor / cap & 0.10 / 0.70 & Dynamic post-warm-up mixtures \\
Warm-up / update interval & 20 / 10 steps & Dynamic schedules \\
Rate-history window & 10 observations & \NashMOPD rate estimation \\
Descent-velocity windows / KL denominator floor & Up to $3\times10$ / 0.15 & \Dthree comparator \\
Softmax temperature / remaining-gap cap & 0.5 / 1.5 & \Dthree comparator \\
Mixture smoothing & 0.25 & \NashMOPD only \\
Batch-level jitter & Uniform $[-0.3,0.3]$ & Domain-sampler conditions \\
Scalar allocation / projection iterations & 80 / 80 & Bounded solver \\
Primary evaluation interval & 20 optimizer steps & Main scheduler comparisons \\
Primary evaluation temperature / top-$p$ & 0.6 / 0.95 & Shared held-out domains \\
\bottomrule
\end{tabular*}}
\setlength{\appendixsettingsheight}{0.6666667\linewidth}
\sbox{\appendixsettingsbox}{\AppendixSettingsTable}
\setlength{\appendixsettingsrowextra}{\dimexpr
  \appendixsettingsheight-\ht\appendixsettingsbox-\dp\appendixsettingsbox\relax}
\divide\appendixsettingsrowextra by 19
\addtolength{\extrarowheight}{\appendixsettingsrowextra}
\sbox{\appendixsettingsbox}{\AppendixSettingsTable}
\setlength{\appendixsettingsrowextra}{\dimexpr
  \appendixsettingsheight-\ht\appendixsettingsbox-\dp\appendixsettingsbox\relax}
\divide\appendixsettingsrowextra by 19
\addtolength{\extrarowheight}{\appendixsettingsrowextra}
\sbox{\appendixsettingsbox}{\AppendixSettingsTable}
\typeout{Appendix settings table: width=\the\wd\appendixsettingsbox;
  height=\the\dimexpr\ht\appendixsettingsbox+\dp\appendixsettingsbox\relax}
\usebox{\appendixsettingsbox}
\endgroup
\caption{Training and scheduling settings for the reported studies.}
\label{tab:implementation-hyperparameters}
\end{minipage}

\end{minipage}
}]

\subsection{State, observations, and update order}
\label{app:controller-state}
The controller stores a domain-specific random row permutation, cursor, prompt
exposure count, initial KL scale, EMA, and recent histories. Each domain's raw
signal is recovered from the dense token-level distillation reward and aggregated
over its non-padding response tokens. Negative raw estimates are clipped to
zero before entering the EMA. The first valid observation initializes the
EMA; later observations use coefficient 0.2. The initial scale is the mean of
the first five raw observations, floored at $10^{-8}$. At the fifth observation,
the gap history is initialized at one; subsequent records use the normalized
remaining gap from the main text. \NashMOPD has no upper clipping of this gap.

A rolling history retains up to 11 exposure--gap pairs. Its oldest and newest
entries provide a gap decrement over up to 10 observed updates. The
nonnegative decrement divided by the exposure difference yields the estimated
exposure-normalized descent velocity, and division by
$\max(g_k,10^{-3})$ gives the exposure-normalized absorption rate. The exposure denominator is floored at one. Rates and gap observations
are updated throughout warm-up, but sampling shares change only once every
domain has a scale and at least two history points. Dynamic updates begin at
step 20 and occur at 10-step intervals. The logged post-update mixture
controls subsequent batches; the current batch has already been sampled.
The teacher signal in row $s$ is measured before optimizer update $s$.

The sampler uses one random-number generator for row permutations,
batch-level jitter, and within-batch ordering. It consumes each domain's current
row permutation before reshuffling that pool. Thus it permits repeated
presentations across training while avoiding independent sampling with
replacement at every draw within a cycle. Exposure counts measure prompt
presentations, including repeats, rather than distinct training examples.
Resume checkpoints retain the sampler's random-number state, current
permutations, cursors, yielded count, histories, EMA, scales, and last update.
Resuming a trainer segment continues this state instead of restarting warm-up.

\subsection{Bounded allocation and comparator details}
\label{app:comparator-details}
\NashMOPD solves for a clipped inverse-rate target by 80 bisection iterations over
$[0,u\max_k\widehat r_k]$, then normalizes its finite-precision result.
It blends this target with the previous mixture using $\alpha=0.25$ and
projects the result onto the capped simplex. The projection itself uses 80
bisection iterations. Share floors and caps apply to the controller's base
mixture after dynamic updates; the jittered probabilities and rounded batch
fractions can exceed those nominal coordinate bounds. The
DeepSeek-R1-Distill-Qwen-1.5B row-proportional warm-up also has an instruction-following base
share below 0.1.

Each batch applies independent multiplicative jitter in $[0.7,1.3]$ to
the scheduled shares and renormalizes them. Largest-remainder rounding
converts the resulting probabilities into integer quotas summing to the
global prompt batch size $B$. Fixed and dynamic schedules use the same
batch construction and random-permutation sampler.

The \Dthree comparator maintains an independent smoothed-KL history of up
to 41 observations. Its remaining-gap signal is
$\bar L_k=\operatorname{clip}(L_k/L_k^{(0)},0,1.5)$. The EMA and
reference scale follow the initialization above. Descent velocity $v_k$ averages the
available 10-step differences from at most three windows, with previous KL
floored at 0.15, and truncates negative descent to zero. Its composite signal multiplies the
remaining-gap and descent-velocity signals, divides by their maximum,
uses softmax temperature 0.5, adds the probability floor, and applies the
same cap. If all signals vanish, it proposes a uniform target. It uses no
additional temporal blend. This baseline uses the original \Dthree
scheduling method, combining remaining gap and descent velocity to
determine the domain sampling probabilities.
Eq.~\eqref{eq:d3-mixture} describes the proposal before the sampling
floor and cap are applied. Both adaptive schedules retain the row-proportional
base mixture during the shared warm-up and use the same batch-level
jitter and integer-quota sampler; only \NashMOPD uses the temporal blend
in Eq.~\eqref{eq:smooth}.

Fixed keeps the usable-row proportions throughout training. In the
DeepSeek-R1-Distill-Qwen-1.5B experiment these proportions are approximately
38.2\% math, 52.7\% code, and 9.1\% instruction following, computed from 17,917, 24,699, and
4,245 loaded rows. Qwen3-8B-Base removes instruction following and recomputes row proportions.

\subsection{Shared training and decoding settings}
\label{app:training-settings}
Schedules within each experiment share initialization, teachers, data,
AdamW settings, and a global batch of 64 with one student rollout per
prompt. Training uses learning rate $10^{-6}$, weight decay 0.01, up to
4,096 prompt tokens and 7,168 response tokens, on A100-SXM4-80GB GPUs.
Held-out decoding uses temperature 0.6 and top-$p$ 0.95, except for
IFBench.
Table~\ref{tab:implementation-hyperparameters} collects the training and
scheduling settings.

\clearpage

\nolinenumbers
\onecolumn
\section{Supplementary Results and Mechanism Analysis}
\label{app:qwen3-supplementary-results}
\label{app:mechanism-protocol}

Figure~\ref{fig:allocation-and-gaps} compares training of the student model
DeepSeek-R1-Distill-Qwen-1.5B over 200 optimizer steps, each representing
one parameter update. The three domains are mathematics, code generation,
and instruction following, which respectively involve solving mathematical
problems, producing programs, and satisfying prompt requirements.
Fixed samples domains in proportion to their training corpus sizes;
\Dthree adjusts sampling using remaining teacher--student divergence and
its recent decrease; \NashMOPD balances predicted proportional gap
reductions across domains. A domain sampling ratio is the scheduled
probability of selecting a prompt from that domain. Reverse
Kullback--Leibler divergence, abbreviated KL, measures the discrepancy
between the student's and the domain teacher's next-token distributions
on student-generated responses. The remaining KL gap is the nonnegative
difference between the smoothed divergence and 10\% of its initial scale,
divided by 90\% of that scale. Smoothing assigns weight 0.2 to the latest
observation and 0.8 to the previous smoothed value; the initial scale is
the mean of the first five nonnegative observations. A gap of one
represents the initial scale, and zero represents the target divergence.

\par\vspace{6pt}
\noindent\begin{minipage}{\textwidth}
\centering
\includegraphics[width=0.66\linewidth]{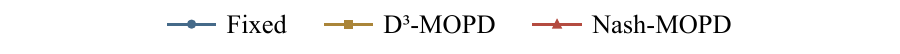}\par
\begin{minipage}[t]{0.33\linewidth}
\centering
\includegraphics[width=\linewidth]{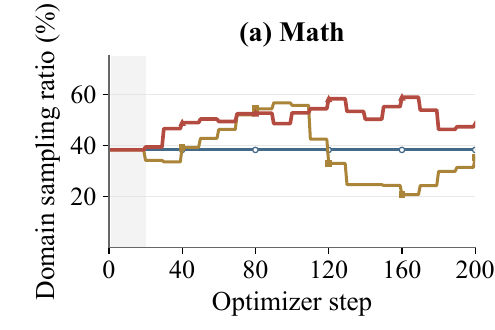}
\end{minipage}\hfill\begin{minipage}[t]{0.33\linewidth}
\centering
\includegraphics[width=\linewidth]{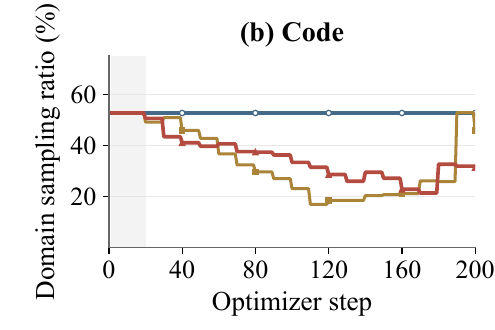}
\end{minipage}\hfill\begin{minipage}[t]{0.33\linewidth}
\centering
\includegraphics[width=\linewidth]{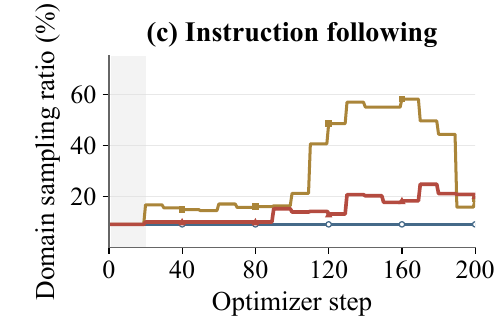}
\end{minipage}\par\vspace{4pt}
\begin{minipage}[t]{0.33\linewidth}
\centering
\includegraphics[width=\linewidth]{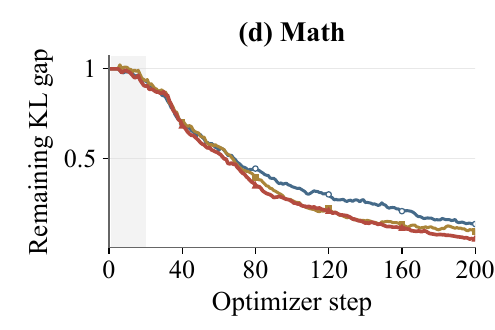}
\end{minipage}\hfill\begin{minipage}[t]{0.33\linewidth}
\centering
\includegraphics[width=\linewidth]{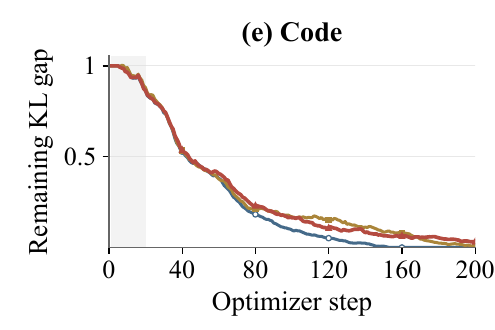}
\end{minipage}\hfill\begin{minipage}[t]{0.33\linewidth}
\centering
\includegraphics[width=\linewidth]{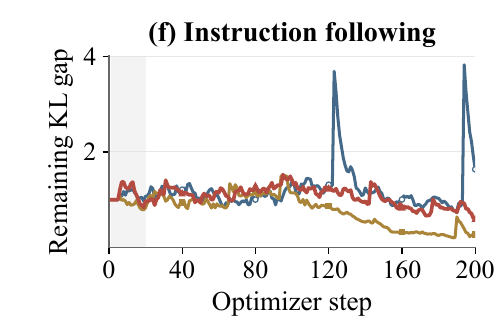}
\end{minipage}
\captionsetup{type=figure,hypcap=false}
\caption{Domain sampling ratios and remaining KL gaps for
 DeepSeek-R1-Distill-Qwen-1.5B through optimizer step 200 under Fixed,
 \Dthree, and \NashMOPD. Columns show mathematics, code generation,
 and instruction following; the upper row shows sampling ratios and
 the lower row shows remaining gaps. Gray shading marks steps 1--20.}
\label{fig:allocation-and-gaps}
\end{minipage}
\par\vspace{6pt}

First, the upper row shows how the methods distribute training prompts.
Fixed maintains ratios of 38.2\%, 52.7\%, and 9.1\% for mathematics,
code generation, and instruction following. In contrast, \NashMOPD
increases the mathematics ratio to 58.3\% by step 120 while reducing
the code generation ratio to 28.6\%. \Dthree makes a larger shift
toward instruction following after step 100, assigning it 58.2\%
at step 160, compared with 18.3\% under \NashMOPD.
Thus, \NashMOPD sustains a larger mathematics allocation, whereas
\Dthree concentrates more prompts on instruction following during
the middle and later stages.

Correspondingly, the lower row shows distinct patterns of gap reduction.
At step 200, the mathematics gap is 0.051 under \NashMOPD,
compared with 0.091 under \Dthree and 0.131 under Fixed.
Code generation shows the reverse ordering: Fixed reaches zero,
while \Dthree and \NashMOPD finish at 0.008 and 0.034.
Meanwhile, the instruction-following gap fluctuates around one during
early training and declines more strongly under \Dthree after its
sampling ratio increases. Its final gap is 0.277 under \Dthree,
0.605 under \NashMOPD, and 1.636 under Fixed. These trajectories
show the domain trade-offs directly: \NashMOPD achieves the smallest
mathematics gap, Fixed the smallest code generation gap, and
\Dthree the smallest instruction-following gap.

\clearpage
Next, Figure~\ref{fig:prediction-diagnostics-macro} compares predictions
of gap reduction over 18 ten-step windows starting at steps
20, 30, $\ldots$, 190. For domain $k$ and starting step $s$,
let $\widehat a_{k,s}$ be the estimated gap decrease per prompt,
computed by dividing the nonnegative gap decrease over up to ten observed
updates by the number of prompt presentations during those updates.
Let $\Delta e_{k,s}$ count prompt presentations over the next ten steps,
including repeated presentations. Holding the starting estimate fixed
gives the predicted reduction
\begin{equation}
 \widehat{\Delta g}_{k,s}=\widehat a_{k,s}\,\Delta e_{k,s}.
\label{eq:app-prediction-check}
\end{equation}
The observed reduction $\Delta g_{k,s}$ is the starting gap minus the
ending gap, so a negative value indicates an increase. Task-macro
correlation is the equal-weight mean of the three domain-specific
Spearman rank correlations between predicted and observed reductions
over the 18 windows. Spearman correlation measures agreement between
their rankings, with larger values indicating closer agreement.
Mean absolute error is the mean of
$|\widehat{\Delta g}_{k,s}-\Delta g_{k,s}|$ over all three domains
and 18 windows; smaller values indicate closer predicted magnitudes.

\par\vspace{6pt}
\noindent\begin{minipage}{\textwidth}
\centering
\includegraphics[width=\linewidth]{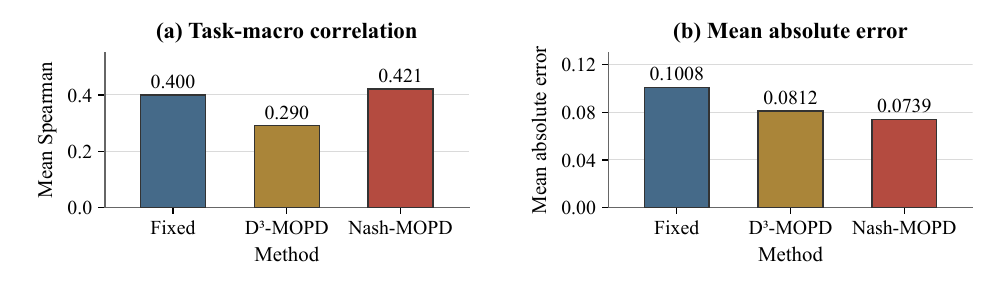}\par
\captionsetup{type=figure,hypcap=false}
\caption{Gap-reduction prediction metrics for DeepSeek-R1-Distill-Qwen-1.5B
 under Fixed, \Dthree, and \NashMOPD across 18 ten-step windows
 from step 20 to step 200. The left panel shows task-macro correlation;
 the right panel shows mean absolute error.}
\label{fig:prediction-diagnostics-macro}
\end{minipage}
\par\vspace{6pt}

Finally, both panels place \NashMOPD ahead of the other methods.
\NashMOPD has the highest task-macro correlation at 0.421,
followed by Fixed at 0.400 and \Dthree at 0.290.
The corresponding correlation advantages are 0.021 and 0.130.
\NashMOPD also has the lowest mean absolute error at 0.0739,
compared with 0.1008 for Fixed and 0.0812 for \Dthree,
reductions of 26.6\% and 8.9\%, respectively.
Together, the panels show that \NashMOPD's gap-reduction predictions
achieve both closer agreement in ranking and smaller magnitude errors
when averaged across the three domains.

\FloatBarrier
\clearpage
\onecolumn
\section{Case Studies}
\label{app:case-studies}
\begingroup
\nolinenumbers
\raggedbottom
\setlength{\parindent}{0pt}
\setlength{\parskip}{2pt}
\newcommand{\casefont}{\fontsize{9.5}{12}\selectfont}
\newcommand{\caseomitted}{\par\vspace{1pt}{\centering\itshape\fontsize{8}{9}\selectfont [\ldots{} intermediate reasoning omitted \ldots{}]\par}\vspace{1pt}}
\newcommand{\casecheck}[1]{\par\smallskip\textbf{Check.} #1}
\newtcolorbox{caserollout}[2][]{enhanced,colback=green!3,colframe=green!45!black,
 boxrule=0.45pt,arc=1pt,title={#2},coltitle=white,toptitle=4pt,bottomtitle=4pt,
 fonttitle=\fontsize{9}{10}\selectfont\bfseries,fontupper=\casefont,
 left=5pt,right=5pt,top=4pt,bottom=4pt,boxsep=0pt,
 before skip=0pt,after skip=5pt,
 before upper={\setlength{\parskip}{2pt}\setlength{\abovedisplayskip}{3pt}\setlength{\belowdisplayskip}{3pt}\setlength{\abovedisplayshortskip}{2pt}\setlength{\belowdisplayshortskip}{2pt}},#1}
\tcbset{case grid/.style={raster columns=2,raster equal height=rows,
 raster column skip=8pt,raster row skip=6pt,
 raster before skip=6pt,raster after skip=0pt}}
\lstset{basicstyle=\ttfamily\fontsize{7.5}{8.5}\selectfont,breaklines=true,
 columns=fullflexible,keepspaces=true,showstringspaces=false,
 aboveskip=3pt,belowskip=3pt}

\subsection{DeepSeek-R1-Distill-Qwen-1.5B: Mathematical Reasoning}
\label{app:case-1p5b-math}
{\fontsize{9}{10.5}\selectfont
\textbf{Question (MATH-500).} A steel sphere with a 3-inch radius is made
by removing metal from the corners of a cube that has the shortest possible
side lengths. How many cubic inches are in the volume of the cube?
\par
\textbf{Reference answer:} $216$.\par}
\begin{tcbraster}[case grid]
\begin{caserollout}{Initial Student (Incorrect)}
The response states that an inscribed sphere ``touches the corners of the
cube,'' and equates its diameter to the cube's space diagonal:
\[
s\sqrt{3}=6,\qquad s=2\sqrt{3}.
\]
\caseomitted
It cubes this edge length:
\[
(2\sqrt{3})^3=24\sqrt{3}.
\]
\textbf{Model answer:} $24\sqrt{3}$.
\casecheck{The sphere must fit inside the cube. Its diameter equals the
shortest cube's edge, rather than its space diagonal. The proposed edge
$2\sqrt{3}<6$ cannot contain the sphere.}
\end{caserollout}
\begin{caserollout}{Fixed (Incorrect)}
The response models the removal of corners using a truncation plane
$x+y+z=3\sqrt{3}$. It then sets a corner-to-plane distance equal to the
sphere's radius:
\[
\left|\frac{3s/2-3\sqrt{3}}{\sqrt{3}}\right|=3.
\]
\caseomitted
It selects $s=4\sqrt{3}$ and computes
\[
(4\sqrt{3})^3=192\sqrt{3}.
\]
\textbf{Model answer:} $192\sqrt{3}$.
\casecheck{The added corner-to-plane constraint does not determine the
smallest enclosing cube. The resulting edge is larger than the required
diameter $6$.}
\end{caserollout}
\begin{caserollout}{\Dthree (Incorrect)}
The response assumes a truncation plane $x+y+z=a$ and uses the distance
from the cube's center $(a/2,a/2,a/2)$ to that plane:
\[
\frac{|3a/2-a|}{\sqrt{3}}=\frac{a}{2\sqrt{3}}=3
\quad\Longrightarrow\quad a=6\sqrt{3}.
\]
\caseomitted
It evaluates $(6\sqrt{3})^3=648\sqrt{3}$.
\par
\textbf{Model answer:} $648\sqrt{3}$.
\casecheck{This assumed plane imposes an unsupported geometric condition.
The shortest enclosing cube is determined by tangency to its six faces,
so its edge is $6$, not $6\sqrt{3}$.}
\end{caserollout}
\begin{caserollout}[colback=red!3,colframe=red!65!black]{\NashMOPD (Correct)}
The response identifies the relevant tangency: ``The sphere is inscribed
inside the cube such that it touches the centers of the faces of the cube.''
It therefore equates the cube's edge to the sphere's diameter:
\[
s=2\times3=6.
\]
\caseomitted
It answers the requested volume, rather than a diagonal or the sphere's volume:
$s^3=6^3=216$.
\par
\textbf{Model answer:} $216$.
\casecheck{The radius-$3$ sphere fits inside a cube of edge $6$ and touches
all six faces. Any smaller edge would exclude part of the sphere.}
\end{caserollout}
\end{tcbraster}

\clearpage
\renewcommand{\casefont}{\fontsize{9}{11}\selectfont}
\subsection{DeepSeek-R1-Distill-Qwen-1.5B: Code Generation}
\label{app:case-1p5b-code}
{\fontsize{9}{10.5}\selectfont
\textbf{Question (LiveCodeBench v6; condensed).} Given an uppercase string
$S$ of length $3$--$100$, count triples $1\leq i<j<k\leq |S|$ such that
$j-i=k-j$ and $S_i=\texttt{A}$, $S_j=\texttt{B}$, $S_k=\texttt{C}$.
Read $S$ from standard input and print the count.
For \texttt{AABCC}, the answer is $2$; for
\texttt{AABAAABBAEDCCCD}, it is $4$.
\par
\textbf{Reference criterion:} pass all 43 saved tests.\par}
\begin{tcbraster}[case grid]
\begin{caserollout}{Initial Student (Incorrect)}
\textit{[Reasoning and explanatory prose omitted.]}
\begin{lstlisting}
S = input().strip()
count = 0
n = len(S)
for d in range(1, n - 1):
    for j in range(1, n - d + 1):
        i = j - d
        k = j + d
        if S[i] == 'A' and S[j] == 'B' and S[k] == 'C':
            count += 1
print(count)
\end{lstlisting}
\textbf{Recorded verdict:} fail; 38/43 tests pass in the local recheck.
\casecheck{The bounds allow $k=n$, outside Python's valid indices.
\texttt{AABAAABBAEDCCCD} raises \texttt{IndexError} instead of printing $4$.}
\end{caserollout}
\begin{caserollout}{Fixed (Incorrect)}
The response repeatedly revisits the indexing and attempts to count the
example \texttt{AABCC}. One pass gives $3$ instead of the reference $2$.
It notices this mismatch: ``So perhaps my approach is still wrong.''
\caseomitted
The saved response ends while reconsidering the loop bounds:
\begin{quote}
``So for len(S)=5, j runs from 1 to4.\par
In j=4:\par
i can''
\end{quote}
\textbf{Recorded verdict:} fail; no complete code block.
\casecheck{The saved generation contains no complete delimited Python
program. It is an unfinished response, so no runnable solution can be
submitted under the benchmark's extraction rule.}
\end{caserollout}
\begin{caserollout}{\Dthree (Incorrect)}

\textit{[Reasoning and explanatory prose omitted.]}
\begin{lstlisting}
s = input().strip()
n = len(s)
count = 0
for j in range(n):
    max_d = n - j - 1
    for d in range(1, max_d + 1):
        i = j - d
        k = j + d
        if s[i] == 'A' and s[j] == 'B' and s[k] == 'C':
            count += 1
print(count)
\end{lstlisting}
\par\smallskip

\textbf{Recorded verdict:} fail; 42/43 tests pass in the local recheck.
\casecheck{The upper bound constrains only the right endpoint.
For $d>j$, $i=j-d$ is negative and Python reads from the end of the string.
The saved test \texttt{AABAAABBAEDCCCD} produces $5$ rather than $4$.}

\end{caserollout}
\begin{caserollout}[colback=red!3,colframe=red!65!black]{\NashMOPD (Correct)}

\textit{[Reasoning and explanatory prose omitted.]}
\begin{lstlisting}
s = input().strip()
count = 0
n = len(s)
for j in range(2, n):
    max_d = min(j - 1, n - j)
    for d in range(1, max_d + 1):
        i = j - d
        k = j + d
        if i < 1 or k > n:
            continue
        if s[i-1] == 'A' and s[j-1] == 'B' and s[k-1] == 'C':
            count += 1
print(count)
\end{lstlisting}
\par\smallskip

\textbf{Recorded verdict:} pass; 43/43 tests pass in the local recheck.
\casecheck{The response uses one-based mathematical indices consistently
and converts them to Python indices at access time.
The bound $d\leq\min(j-1,n-j)$ keeps both endpoints inside the string.
It prints $4$ on the same counterexample.}

\end{caserollout}
\end{tcbraster}

\clearpage
\renewcommand{\casefont}{\fontsize{9.5}{12}\selectfont}
\subsection{DeepSeek-R1-Distill-Qwen-1.5B: Instruction Following}
\label{app:case-1p5b-if}
{\fontsize{9}{10.5}\selectfont
\textbf{Question (IFBench; original English prompt).}
Generate 4 multiple choice questions with 5 options each about ``20th century art history''.
Each question should start with the label ``Question''. The questions should get
progressively longer. Do not provide an explanation.
\par
\textbf{Reference criterion:} the official
\texttt{custom:mcq\_count\_length} check. Length is the character count of
question text before its options.\par}
\begin{tcbraster}[case grid]
\begin{caserollout}{Initial Student (Incorrect)}
\textit{[Reasoning and the five options per question omitted; all question texts retained.]}
\par\smallskip
\textbf{Question 1:} The Great Depression and World War II had a significant impact on 20th century art. Which of the following best describes this impact?
\par\smallskip
\textbf{Question 2:} The 1950s and 1960s were a time of digital art and computer use. Which movement emerged during this period?
\par\smallskip
\textbf{Question 3:} The 1970s and 1980s were a time of experimentation and new art movements. Which movement emerged during this period?
\par\smallskip
\textbf{Question 4:} The 1990s and 2000s brought about the digital revolution and the rise of digital art. Which of the following best describes this period?
\par\smallskip
\textbf{Recorded verdict:} strict fail; loose fail.
\casecheck{The question lengths are $134,107,116,136$ characters. The second question is shorter than the first, so the required increasing-length sequence fails.}
\end{caserollout}
\begin{caserollout}{Fixed (Incorrect)}
\textit{[Reasoning and the five options per question omitted; all question texts retained.]}
\par\smallskip
\textbf{Question 1:} Which of the following best describes the origins of 20th century art?
\par\smallskip
\textbf{Question 2:} What was a key characteristic of Expressionism in 20th century art?
\par\smallskip
\textbf{Question 3:} How did the Great Depression and World War II shape 20th century art?
\par\smallskip
\textbf{Question 4:} What is a major concept in 20th century art history?
\par\smallskip
\textbf{Recorded verdict:} strict fail; loose fail.
\casecheck{The question lengths are $70,67,69,52$ characters. They decrease at the first transition and again at the last transition.}
\end{caserollout}
\begin{caserollout}{\Dthree (Incorrect)}
\textit{[Reasoning and the five options per question omitted; all question texts retained.]}
\par\smallskip
\textbf{Question} Which artist is best known for creating abstract geometric patterns in the 20th century?
\par\smallskip
\textbf{Question} Which period in the 20th century saw the most significant contributions to modern art?
\par\smallskip
\textbf{Question} Which of the following artists is associated with the use of color in their work during the 20th century?
\par\smallskip
\textbf{Question} Which of the following movements was most influential in shaping 20th century art?
\par\smallskip
\textbf{Recorded verdict:} strict fail; loose fail.
\casecheck{The question lengths are $88,86,105,82$ characters and are not strictly increasing. It also omits question numbers expected by the official parser.}
\end{caserollout}
\begin{caserollout}[colback=red!3,colframe=red!65!black]{\NashMOPD (Correct)}
\textit{[Reasoning and the five options per question omitted; all question texts retained.]}
\par\smallskip
\textbf{Question 1:} What was the first major movement in 20th century art history?
\par\smallskip
\textbf{Question 2:} Which of the following is a key figure in the Realism movement?
\par\smallskip
\textbf{Question 3:} What theory emerged in the 20th century that influenced modern art?
\par\smallskip
\textbf{Question 4:} Which of the following is a significant theory in 20th century art history?
\par\smallskip
\textbf{Recorded verdict:} strict pass; loose pass.
\casecheck{The question lengths are $62,63,67,75$ characters, so all three transitions increase. The original response has four numbered questions, each with five options, and no explanatory prose. This verdict checks the requested format and lengths, rather than the factual quality of the quiz.}
\end{caserollout}
\end{tcbraster}

\clearpage
\subsection{Qwen3-8B-Base: Mathematical Reasoning}
\label{app:case-8b-math}
{\fontsize{9}{10.5}\selectfont
\textbf{Question (MATH-500).} How many of the same digits are found in
the base 7 and base 8 representations of $629_{10}$? For example,
$121_3$ and $413_5$ would have one digit in common.
\par
\textbf{Reference answer:} $3$.\par}
\begin{tcbraster}[case grid]
\begin{caserollout}{Initial Student (Incorrect)}
The response describes repeated division and supplies a conversion
function, then claims the following output:
\begin{lstlisting}
([1, 6, 0, 2], [1, 1, 4, 5], {1}, 1)
\end{lstlisting}
\caseomitted
It concludes that the representations are $1602_7$ and $1145_8$,
and that only the digit $1$ appears in both.
\par
\textbf{Model answer:} $1$.
\casecheck{The claimed conversions are incorrect:
$1602_7=639$ and $1145_8=613$.
The actual representations are $1556_7$ and $1165_8$.}
\end{caserollout}
\begin{caserollout}{Fixed (Incorrect)}
The response correctly obtains $1556_7$ and $1165_8$, but compares
every occurrence of a digit in one representation with every occurrence
in the other. It lists five matches:
\begin{quote}
``1 vs 1'' (twice), ``5 vs 5'' (twice), and ``6 vs 6''.
\end{quote}
\caseomitted
It decides that matches at different positions should count separately.
\par
\textbf{Model answer:} $5$.
\casecheck{The question counts distinct shared digits, as its example
indicates. Repeated occurrences of $1$ and $5$ must not increase the count.}
\end{caserollout}
\begin{caserollout}{\Dthree (Incorrect)}
The response also correctly converts the number to $1556_7$ and $1165_8$.
It aligns their positions and compares $1$ with $1$, $5$ with $1$,
$5$ with $6$, and $6$ with $5$.
\caseomitted
It concludes: ``Only the first digit is the same.''
\par
\textbf{Model answer:} $1$.
\casecheck{Shared digits need not occur in the same position.
Digits $5$ and $6$ occur in both representations, so positional comparison
misses two valid shared digits.}
\end{caserollout}
\begin{caserollout}[colback=red!3,colframe=red!65!black]{\NashMOPD (Correct)}
The response computes $629_{10}=1556_7=1165_8$.
It initially considers a positional comparison, then revisits the meaning:
``Alternatively, perhaps I should consider that the digits are the same
regardless of their positions, just looking for any matching digits.''
\caseomitted
It identifies $1$, $5$, and $6$ in both representations, and explicitly
counts each digit only once despite repeated occurrences.
\par
\textbf{Model answer:} $3$.
\casecheck{The distinct digit sets are both $\{1,5,6\}$.
Their intersection has size $3$, matching the question's example.}
\end{caserollout}
\end{tcbraster}

\clearpage
\renewcommand{\casefont}{\fontsize{9}{11}\selectfont}
\subsection{Qwen3-8B-Base: Code Generation}
\label{app:case-8b-code}
{\fontsize{9}{10.5}\selectfont
\textbf{Question (LiveCodeBench v6; condensed).}
Implement \texttt{Solution.largestInteger(nums, k)}.
An integer is ``almost missing'' if it appears in exactly one contiguous
subarray of size $k$. Return the largest such integer, or $-1$ if none
exists. Constraints: $1\leq |\texttt{nums}|\leq50$,
$0\leq\texttt{nums}[i]\leq50$, and $1\leq k\leq|\texttt{nums}|$.
\par
\textbf{Reference criterion:} pass all 34 saved tests.
Repeated occurrences inside one subarray count as one containing subarray.\par}
\begin{tcbraster}[case grid]
\begin{caserollout}{Initial Student (Incorrect)}
Its first code block accumulates each occurrence:
\textit{[Surrounding code and prose omitted.]}
\begin{lstlisting}
for subarray in subarrays:
    for num in subarray:
        count[num] += 1
\end{lstlisting}
The last code block contains only example calls:
\begin{lstlisting}
solution = Solution()
print(solution.largestInteger([3, 9, 2, 1, 7], 3))  # Output: 7
print(solution.largestInteger([3, 9, 7, 2, 1, 7], 4))  # Output: 3
print(solution.largestInteger([0, 0], 1))  # Output: -1
\end{lstlisting}
\textbf{Recorded verdict:} fail; 0/34 tests pass in the local recheck.
\casecheck{The benchmark extracts the last complete code block,
which has no class definition and raises \texttt{NameError}.
The earlier implementation also counts repeated elements within one
subarray multiple times.}
\end{caserollout}
\begin{caserollout}{Fixed (Incorrect)}
\textit{[Imports, signature, comments and prose omitted.]}
\begin{lstlisting}
subarrays = []
for i in range(len(nums) - k + 1):
    subarrays.append(nums[i:i + k])
from collections import defaultdict
count_dict = defaultdict(int)
for subarray in subarrays:
    for num in subarray:
        count_dict[num] += 1
largest_integer = -1
for num in sorted(count_dict.keys(), reverse=True):
    if count_dict[num] == 1:
        largest_integer = num
        break
return largest_integer
\end{lstlisting}
\textbf{Recorded verdict:} fail; 30/34 tests pass in the local recheck.
\casecheck{For \texttt{nums = [0] * 50} and $k=50$, there is only one
subarray. The correct answer is $0$, but the code counts 50 occurrences
and returns $-1$.}
\end{caserollout}
\begin{caserollout}{\Dthree (Incorrect)}

\textit{[Imports, signature, comments and prose omitted.]}
\begin{lstlisting}
subarrays = []
for i in range(len(nums) - k + 1):
    subarrays.append(nums[i:i + k])
count_dict = {}
for num in nums:
    count_dict[num] = 0
for subarray in subarrays:
    for num in subarray:
        if num in count_dict:
            count_dict[num] += 1
max_num = -1
for num, count in count_dict.items():
    if count == 1 and num > max_num:
        max_num = num
return max_num
\end{lstlisting}
\par\smallskip

\textbf{Recorded verdict:} fail; 30/34 tests pass in the local recheck.
\casecheck{The membership check tests whether a dictionary entry exists;
it does not deduplicate elements within a subarray.
Thus the same saved counterexample, \texttt{nums = [0] * 50}, $k=50$,
returns $-1$ instead of $0$.}

\end{caserollout}
\begin{caserollout}[colback=red!3,colframe=red!65!black]{\NashMOPD (Correct)}

\textit{[Imports, signature, comments and prose omitted.]}
\begin{lstlisting}
subarrays = []
for i in range(len(nums) - k + 1):
    subarrays.append(nums[i:i + k])
count_dict = {}
for num in nums:
    count = 0
    for subarray in subarrays:
        if num in subarray:
            count += 1
    count_dict[num] = count
result = -1
for num, count in count_dict.items():
    if count == 1 and num > result:
        result = num
return result
\end{lstlisting}
\par\smallskip

\textbf{Recorded verdict:} pass; 34/34 tests pass in the local recheck.
\casecheck{Each subarray contributes at most one count for a given
integer via \texttt{num in subarray}.
It returns $0$ for the all-zero counterexample, because $0$ belongs to
exactly one size-50 subarray. It then selects the largest integer whose
containing-subarray count equals one.}

\end{caserollout}
\end{tcbraster}
\endgroup


\begin{thebibliography}{30}
\providecommand{\natexlab}[1]{#1}

\bibitem[{Agarwal et~al.(2024)Agarwal, Vieillard, Zhou, Stanczyk, Ramos~Garea,
  Geist, and Bachem}]{agarwal2024gkd}
Rishabh Agarwal, Nino Vieillard, Yongchao Zhou, Piotr Stanczyk, Sabela
  Ramos~Garea, Matthieu Geist, and Olivier Bachem. 2024.
\newblock \href
  {https://proceedings.iclr.cc/paper_files/paper/2024/hash/5be69a584901a26c521c2b51e40a4c20-Abstract-Conference.html}
  {On-policy distillation of language models: Learning from self-generated
  mistakes}.
\newblock In \emph{International Conference on Learning Representations}.

\bibitem[{Austin et~al.(2021)Austin, Odena, Nye, Bosma, Michalewski, Dohan,
  Jiang, Cai, Terry, Le, and Sutton}]{austin2021programsynthesis}
Jacob Austin, Augustus Odena, Maxwell Nye, Maarten Bosma, Henryk Michalewski,
  David Dohan, Ellen Jiang, Carrie Cai, Michael Terry, Quoc Le, and Charles
  Sutton. 2021.
\newblock \href {https://arxiv.org/abs/2108.07732} {Program synthesis with
  large language models}.
\newblock \emph{arXiv preprint arXiv:2108.07732}.

\bibitem[{Chen et~al.(2021)Chen, Tworek, Jun et~al.}]{chen2021evaluatingcode}
Mark Chen, Jerry Tworek, Heewoo Jun, et~al. 2021.
\newblock \href {https://doi.org/10.48550/arXiv.2107.03374} {Evaluating large
  language models trained on code}.
\newblock \emph{arXiv preprint arXiv:2107.03374}.

\bibitem[{Chen et~al.(2018)Chen, Badrinarayanan, Lee, and
  Rabinovich}]{chen2018gradnorm}
Zhao Chen, Vijay Badrinarayanan, Chen-Yu Lee, and Andrew Rabinovich. 2018.
\newblock \href {https://proceedings.mlr.press/v80/chen18a.html} {{GradNorm}:
  Gradient normalization for adaptive loss balancing in deep multitask
  networks}.
\newblock In \emph{Proceedings of the 35th International Conference on Machine
  Learning}, volume~80 of \emph{Proceedings of Machine Learning Research}.
  PMLR.

\bibitem[{{DeepSeek-AI} et~al.(2026)}]{deepseek2026v4}
{DeepSeek-AI} et~al. 2026.
\newblock \href {https://arxiv.org/abs/2606.19348} {{DeepSeek-V4}: Towards
  highly efficient million-token context intelligence}.
\newblock \emph{arXiv preprint arXiv:2606.19348}.

\bibitem[{Feng et~al.(2025)Feng, Xue, Liu, and An}]{feng2025gigpo}
Lang Feng, Zhenghai Xue, Tingcong Liu, and Bo~An. 2025.
\newblock \href {https://doi.org/10.52202/085713-1544} {Group-in-group policy
  optimization for {LLM} agent training}.
\newblock In \emph{Advances in Neural Information Processing Systems},
  volume~38.

\bibitem[{Fu et~al.(2026)Fu, He, Zuo, Huang, Zhang, Xiao, Qian, Luo, Gao, Wang,
  Liu, Ding, and Xiao}]{fu2026oneexampleopd}
Zixuan Fu, Bingxiang He, Yuxin Zuo, Haohuan Huang, Jinqian Zhang, Ruhang Xiao,
  Cheng Qian, Qinyu Luo, Huan-ang Gao, Yudong Wang, Zhiyuan Liu, Ning Ding, and
  Chaojun Xiao. 2026.
\newblock \href {https://arxiv.org/abs/2609.04172} {Rethinking on-policy
  distillation of large language models {II}: One training example}.
\newblock \emph{arXiv preprint arXiv:2609.04172}.

\bibitem[{Gu et~al.(2024)Gu, Dong, Wei, and Huang}]{gu2024minillm}
Yuxian Gu, Li~Dong, Furu Wei, and Minlie Huang. 2024.
\newblock \href
  {https://proceedings.iclr.cc/paper_files/paper/2024/hash/8ac015d409635f196f9e3e9dcfb9a94e-Abstract-Conference.html}
  {{MiniLLM}: Knowledge distillation of large language models}.
\newblock In \emph{International Conference on Learning Representations}.

\bibitem[{Guo et~al.(2025)Guo, Yang, Zhang et~al.}]{deepseek2025r1}
Daya Guo, Dejian Yang, Haowei Zhang, et~al. 2025.
\newblock \href {https://doi.org/10.1038/s41586-025-09422-z} {{DeepSeek-R1}
  incentivizes reasoning in {LLMs} through reinforcement learning}.
\newblock \emph{Nature}, 645(8081):633--638.

\bibitem[{Guo et~al.(2018)Guo, Haque, Huang, Yeung, and
  Fei-Fei}]{guo2018dynamic}
Michelle Guo, Albert Haque, De-An Huang, Serena Yeung, and Li~Fei-Fei. 2018.
\newblock \href
  {https://openaccess.thecvf.com/content_ECCV_2018/html/Michelle_Guo_Focus_on_the_ECCV_2018_paper.html}
  {Dynamic task prioritization for multitask learning}.
\newblock In \emph{Proceedings of the European Conference on Computer Vision}.

\bibitem[{He et~al.(2026)He, Qu, Liu, Chen, Zuo, Qian, Zhang, Chen, Xiao, Cui,
  Ding, and Liu}]{he2025justrl}
Bingxiang He, Zekai Qu, Zeyuan Liu, Yinghao Chen, Yuxin Zuo, Cheng Qian, Kaiyan
  Zhang, Weize Chen, Chaojun Xiao, Ganqu Cui, Ning Ding, and Zhiyuan Liu. 2026.
\newblock \href {https://iclr-blogposts.github.io/2026/blog/2026/justrl/}
  {{JustRL}: Scaling a 1.5b {LLM} with a simple {RL} recipe}.
\newblock In \emph{{ICLR} Blogposts 2026}.

\bibitem[{He et~al.(2024)He, Jin, Wang, Bi, Mandyam, Zhang, Zhu, Li, Xu, Lv,
  Bhosale, Zhu, Sankararaman, Helenowski, Kambadur, Tayade, Ma, Fang, and
  Wang}]{he2024multiif}
Yun He, Di~Jin, Chaoqi Wang, Chloe Bi, Karishma Mandyam, Hejia Zhang, Chen Zhu,
  Ning Li, Tengyu Xu, Hongjiang Lv, Shruti Bhosale, Chenguang Zhu,
  Karthik~Abinav Sankararaman, Eryk Helenowski, Melanie Kambadur, Aditya
  Tayade, Hao Ma, Han Fang, and Sinong Wang. 2024.
\newblock \href {https://arxiv.org/abs/2410.15553} {{Multi-IF}: Benchmarking
  {LLM}s on multi-turn and multilingual instructions following}.
\newblock \emph{arXiv preprint arXiv:2410.15553}.

\bibitem[{Hendrycks et~al.(2021)Hendrycks, Burns, Kadavath, Arora, Basart,
  Tang, Song, and Steinhardt}]{hendrycks2021math}
Dan Hendrycks, Collin Burns, Saurav Kadavath, Akul Arora, Steven Basart, Eric
  Tang, Dawn Song, and Jacob Steinhardt. 2021.
\newblock \href
  {https://datasets-benchmarks-proceedings.neurips.cc/paper/2021/hash/be83ab3ecd0db773eb2dc1b0a17836a1-Abstract-round2.html}
  {Measuring mathematical problem solving with the {MATH} dataset}.
\newblock In \emph{Proceedings of the Neural Information Processing Systems
  Track on Datasets and Benchmarks 1}.

\bibitem[{Jain et~al.(2025)Jain, Han, Gu, Li, Yan, Zhang, Wang, Solar-Lezama,
  Sen, and Stoica}]{jain2025livecodebench}
Naman Jain, King Han, Alex Gu, Wen-Ding Li, Fanjia Yan, Tianjun Zhang, Sida
  Wang, Armando Solar-Lezama, Koushik Sen, and Ion Stoica. 2025.
\newblock \href
  {https://proceedings.iclr.cc/paper_files/paper/2025/hash/94074dd5a072d28ff75a76dabed43767-Abstract-Conference.html}
  {{LiveCodeBench}: Holistic and contamination free evaluation of large
  language models for code}.
\newblock In \emph{International Conference on Learning Representations}.

\bibitem[{{Kimi Team}(2026)}]{kimi2026k3}
{Kimi Team}. 2026.
\newblock \href {https://arxiv.org/abs/2607.24653} {{Kimi K3}: Open frontier
  intelligence}.
\newblock \emph{arXiv preprint arXiv:2607.24653}.

\bibitem[{Li et~al.(2023)Li, Fu, Zhang, Huang, Sun, Lyu, Liu, Jin, and
  Li}]{li2023taco}
Rongao Li, Jie Fu, Bo-Wen Zhang, Tao Huang, Zhihong Sun, Chen Lyu, Guang Liu,
  Zhi Jin, and Ge~Li. 2023.
\newblock \href {https://arxiv.org/abs/2312.14852} {{TACO}: Topics in
  algorithmic {CO}de generation dataset}.
\newblock \emph{arXiv preprint arXiv:2312.14852}.

\bibitem[{Liu et~al.(2023)Liu, Xia, Wang, and Zhang}]{liu2023evalplus}
Jiawei Liu, Chunqiu~Steven Xia, Yuyao Wang, and Lingming Zhang. 2023.
\newblock \href {https://openreview.net/forum?id=1qvx610Cu7} {Is your code
  generated by {ChatGPT} really correct? rigorous evaluation of large language
  models for code generation}.
\newblock In \emph{Advances in Neural Information Processing Systems}.

\bibitem[{{LLM-Core Xiaomi}(2026)}]{xiaomi2026mimov26}
{LLM-Core Xiaomi}. 2026.
\newblock \href
  {https://huggingface.co/XiaomiMiMo/MiMo-V2.6-Pro-RL/blob/main/MiMo_V2_6_technical_report.pdf}
  {{MiMo-V2.6}: Scaling reinforcement learning towards self-improvement}.
\newblock Technical report, Xiaomi.

\bibitem[{Ma et~al.(2026)Ma, Wei, Zhao, Zhang, Xiao, Li, Yang, Gao, Wang, Li,
  Dong, Sui, and Luo}]{ma2026mopd}
Wenhan Ma, Jianyu Wei, Liang Zhao, Hailin Zhang, Bangjun Xiao, Lei Li, Qibin
  Yang, Bofei Gao, Yudong Wang, Rang Li, Jinhao Dong, Zhifang Sui, and Fuli
  Luo. 2026.
\newblock \href {https://arxiv.org/abs/2606.30406} {{MOPD}: Multi-teacher
  on-policy distillation for capability integration in {LLM} post-training}.
\newblock \emph{arXiv preprint arXiv:2606.30406}.

\bibitem[{Meng et~al.(2026)Meng, Gan, Li, Wang, Ding, Gou, Wu, and
  Zhao}]{meng2026unbiasedtopk}
Linjian Meng, Siyuan Gan, YuHan Li, Xiran Wang, Ziyang Ding, Ditang Gou, Yiming
  Wu, and Zhen Zhao. 2026.
\newblock \href {https://arxiv.org/abs/2609.34447} {Unbiased top-$k$ estimation
  for on-policy distillation}.
\newblock \emph{arXiv preprint arXiv:2609.34447}.

\bibitem[{Munos et~al.(2024)Munos, Valko, Calandriello, Gheshlaghi~Azar,
  Rowland, Guo, Tang, Geist, Mesnard, Fiegel, Michi, Selvi, Girgin, Momchev,
  Bachem, Mankowitz, Precup, and Piot}]{munos2024nash}
R{\'e}mi Munos, Michal Valko, Daniele Calandriello, Mohammad Gheshlaghi~Azar,
  Mark Rowland, Zhaohan~Daniel Guo, Yunhao Tang, Matthieu Geist, Thomas
  Mesnard, C{\^o}me Fiegel, Andrea Michi, Marco Selvi, Sertan Girgin, Nikola
  Momchev, Olivier Bachem, Daniel~J. Mankowitz, Doina Precup, and Bilal Piot.
  2024.
\newblock \href {https://proceedings.mlr.press/v235/munos24a.html} {{Nash
  Learning from Human Feedback}}.
\newblock In \emph{Proceedings of the 41st International Conference on Machine
  Learning}, volume 235 of \emph{Proceedings of Machine Learning Research}.
  PMLR.

\bibitem[{Navon et~al.(2022)Navon, Shamsian, Achituve, Maron, Kawaguchi,
  Chechik, and Fetaya}]{navon2022bargaining}
Aviv Navon, Aviv Shamsian, Idan Achituve, Haggai Maron, Kenji Kawaguchi, Gal
  Chechik, and Ethan Fetaya. 2022.
\newblock \href {https://proceedings.mlr.press/v162/navon22a.html} {Multi-task
  learning as a bargaining game}.
\newblock In \emph{Proceedings of the 39th International Conference on Machine
  Learning}, volume 162 of \emph{Proceedings of Machine Learning Research}.
  PMLR.

\bibitem[{Pyatkin et~al.(2025)Pyatkin, Malik, Graf, Ivison, Huang, Dasigi,
  Lambert, and Hajishirzi}]{pyatkin2025ifbench}
Valentina Pyatkin, Saumya Malik, Victoria Graf, Hamish Ivison, Shengyi Huang,
  Pradeep Dasigi, Nathan Lambert, and Hanna Hajishirzi. 2025.
\newblock \href {https://doi.org/10.52202/085713-1645} {Generalizing verifiable
  instruction following}.
\newblock In \emph{Advances in Neural Information Processing Systems 38,
  Datasets and Benchmarks Track}.

\bibitem[{Sener and Koltun(2018)}]{sener2018mto}
Ozan Sener and Vladlen Koltun. 2018.
\newblock \href
  {https://proceedings.neurips.cc/paper/2018/hash/432aca3a1e345e339f35a30c8f65edce-Abstract.html}
  {Multi-task learning as multi-objective optimization}.
\newblock In \emph{Advances in Neural Information Processing Systems},
  volume~31.

\bibitem[{Sun et~al.(2026)Sun, Zhang, Zhao, Li, Chuan, Deng, Zhan, Chen, Hu,
  and Zhang}]{sun2026d3mopd}
Zechen Sun, Zhiwei Zhang, Fei Zhao, Juntao Li, Mu~Chuan, Huayu Deng, Guojian
  Zhan, Wenliang Chen, Yao Hu, and Min Zhang. 2026.
\newblock \href {https://arxiv.org/abs/2608.24987} {{D$^3$-MOPD}: Dynamic
  domain scheduling for efficient multi-teacher distillation}.
\newblock \emph{arXiv preprint arXiv:2608.24987}.

\bibitem[{Sundararaman et~al.(2021)Sundararaman, Tsai, Lee, Turc, and
  Carin}]{sundararaman2021sampling}
Dhanasekar Sundararaman, Henry Tsai, Kuang-Huei Lee, Iulia Turc, and Lawrence
  Carin. 2021.
\newblock \href {https://doi.org/10.18653/v1/2021.findings-emnlp.375} {Learning
  task sampling policy for multitask learning}.
\newblock In \emph{Findings of the Association for Computational Linguistics:
  EMNLP 2021}. Association for Computational Linguistics.

\bibitem[{Xie et~al.(2023)Xie, Pham, Dong, Du, Liu, Lu, Liang, Le, Ma, and
  Yu}]{xie2023doremi}
Sang~Michael Xie, Hieu Pham, Xuanyi Dong, Nan Du, Hanxiao Liu, Yifeng Lu, Percy
  Liang, Quoc~V. Le, Tengyu Ma, and Adams~Wei Yu. 2023.
\newblock \href {https://doi.org/10.52202/075280-3059} {{DoReMi}: Optimizing
  data mixtures speeds up language model pretraining}.
\newblock In \emph{Advances in Neural Information Processing Systems 36}.

\bibitem[{Yang et~al.(2025)}]{yang2025qwen3}
An~Yang et~al. 2025.
\newblock \href {https://arxiv.org/abs/2505.09388} {{Qwen3} technical report}.
\newblock \emph{arXiv preprint arXiv:2505.09388}.

\bibitem[{Yu et~al.(2025)Yu, Zhang, Zhu, Yuan, Zuo, Yue, Dai, Fan, Liu, Liu,
  Liu, Liu, Lin, Lin, Ma, Sheng, Tong, Zhang, Zhang, Zhang, Zhang, Zhu, Zhu,
  Chen, Chen, Wang, Yu, Song, Wei, Zhou, Liu, Ma, Zhang, Yan, Wu, and
  Wang}]{yu2025dapo}
Qiying Yu, Zheng Zhang, Ruofei Zhu, Yufeng Yuan, Xiaochen Zuo, Yu~Yue, Weinan
  Dai, Tiantian Fan, Gaohong Liu, Juncai Liu, LingJun Liu, Xin Liu, Haibin Lin,
  Zhiqi Lin, Bole Ma, Guangming Sheng, Yuxuan Tong, Chi Zhang, Mofan Zhang, and
  17 others. 2025.
\newblock \href {https://doi.org/10.52202/085713-3775} {{DAPO}: An open-source
  {LLM} reinforcement learning system at scale}.
\newblock In \emph{Advances in Neural Information Processing Systems 38}.

\bibitem[{Zeng et~al.(2026)}]{zeng2026glm5}
Aohan Zeng et~al. 2026.
\newblock \href {https://arxiv.org/abs/2602.15763} {{GLM-5}: from vibe coding
  to agentic engineering}.
\newblock \emph{arXiv preprint arXiv:2602.15763}.

\end{thebibliography}
\end{document}